\documentclass{article}

\usepackage[preprint]{neurips_2025}

\usepackage[citecolor=black]{hyperref}
\usepackage[utf8]{inputenc} 
\usepackage[T1]{fontenc}    
\usepackage{hyperref}       
\usepackage{url}            
\usepackage{booktabs}       
\usepackage{amsfonts}       
\usepackage{nicefrac}       
\usepackage{microtype}      
\usepackage{xcolor}         
\usepackage{url}
\usepackage{amsmath}
\usepackage{amssymb}
\usepackage{mathtools}
\usepackage{amsthm}
\usepackage{latexsym}
\usepackage{url}
\usepackage{latexsym}
\usepackage{subfigure}
\usepackage{graphicx}
\usepackage{amsmath}
\usepackage{mathrsfs}
\usepackage{amsfonts} 
\usepackage{bbm}
\usepackage{algorithm2e}
\usepackage{color}
\usepackage{multirow}
\usepackage{booktabs}
\usepackage{arydshln}
\usepackage{enumitem}
\usepackage{amssymb}
\usepackage{tabularx}
\usepackage{makecell}
\usepackage{fancyhdr}
\usepackage{wrapfig}
\hypersetup{
	colorlinks=true,
	linkcolor=black,
	filecolor=black,      
	urlcolor=black,
}

\newtheorem{theorem}{Theorem}
\newtheorem{proposition}{Proposition}
\newtheorem{lemma}{Lemma}

\newtheorem{definition}{Definition}
\newtheorem{assumption}{Assumption}
\newtheorem{remark}{Remark}

\title{Online Knowledge Distillation with Reward Guidance}

\author{%
	Chen Jia \\
	SI-TECH Information Technology \\
	\texttt{jiachenwestlake@gmail.com} \\
}

\begin{document}

\maketitle

\begin{abstract}
This work studies knowledge distillation (KD) for large language models (LLMs) through preference optimization. We propose a reward-guided imitation learning framework for sequential KD, formulating a min-max optimization problem between the policy and reward model (RM) to minimize the performance gap between the student and teacher policies. Specifically, the reward optimization is constrained to achieve near-optimality within a confidence set for preference alignment. For preference data construction, we explore both offline and online preference-based KD. Additionally, we reformulate the RM using the $Q$-value function and extend the framework to white-box KD, where the teacher policy's predicted probabilities are accessible. Theoretical analysis and empirical results demonstrate the effectiveness of the proposed framework.
\end{abstract}

\section{Introduction}
Large language models (LLMs) such as GPT-4 \cite{achiam2023gpt} and PaLM \cite{anil2023palm} have demonstrated impressive capabilities across a wide range of natural language processing (NLP) tasks. Their strong generalization performance is largely attributed to large-scale pretraining and increased model capacity. However, these benefits come at the cost of substantial computational and memory demands, which pose significant challenges for real-world deployment, particularly in resource-constrained settings.

To address these challenges, knowledge distillation (KD) \cite{hinton2015distilling} has become a widely adopted technique to compress large teacher models into smaller, more efficient student models, while aiming to preserve much of the teacher’s performance. KD methods can be broadly categorized into black-box and white-box approaches, depending on the accessibility of the teacher model’s internal outputs. In black-box KD, the student learns by mimicking the teacher’s outputs using maximum likelihood estimation (MLE) \cite{kim2016sequence}. In white-box KD, the student has access to the teacher’s output distributions, allowing for more informative training through distribution matching \cite{hinton2015distilling,lin2020autoregressive,wen2023f,gu2023minillm,agarwal2024policy}.

Despite their success, existing KD methods face two major challenges. First, the large capacity gap between teacher and student models makes it difficult for the student to accurately replicate the teacher’s outputs, particularly in white-box KD. Second, the teacher’s outputs may themselves be suboptimal for specific downstream tasks, making it undesirable to blindly mimic them.

To overcome these limitations, we propose a novel preference-based knowledge distillation \textbf{(PbKD)} framework as shown in Figure \ref{fig:overall}. Instead of directly cloning the teacher’s behavior, PbKD formulates KD as a reward-guided imitation learning problem. Our method optimizes the performance gap between the student and teacher policies through a min-max game between the student and a reward model (RM). The RM is constrained within a confidence set constructed from preference data and is optimized to maximize the distinguishability between the student and teacher, guiding the student to improve accordingly.

We further develop an online PbKD variant that iteratively augments the preference dataset with new samples collected from the evolving student policy, enabling continual refinement of the RM and more accurate alignment. In the white-box setting, we extend PbKD by reformulating the RM using $Q$-functions to more effectively leverage the teacher’s output probabilities.

We provide theoretical guarantees for both the offline and online settings. For offline PbKD, we derive a suboptimality bound with a convergence rate of $\mathcal{O}(\sqrt{\log(N/\delta)/N})$, and for online PbKD, we prove a regret bound of order $\mathcal{O}(\sqrt{ T \log T \log(T/\delta) })$, where $N$ is the number of preference samples, $d$ is the feature dimension, $T$ is the number of iterations, and $\delta$ is the confidence level. Extensive empirical evaluations across five black-box and five white-box KD benchmarks confirm the effectiveness of our approach, outperforming prior methods.

\begin{figure}[t!]
	\centering	
	\includegraphics[width=1.\linewidth]{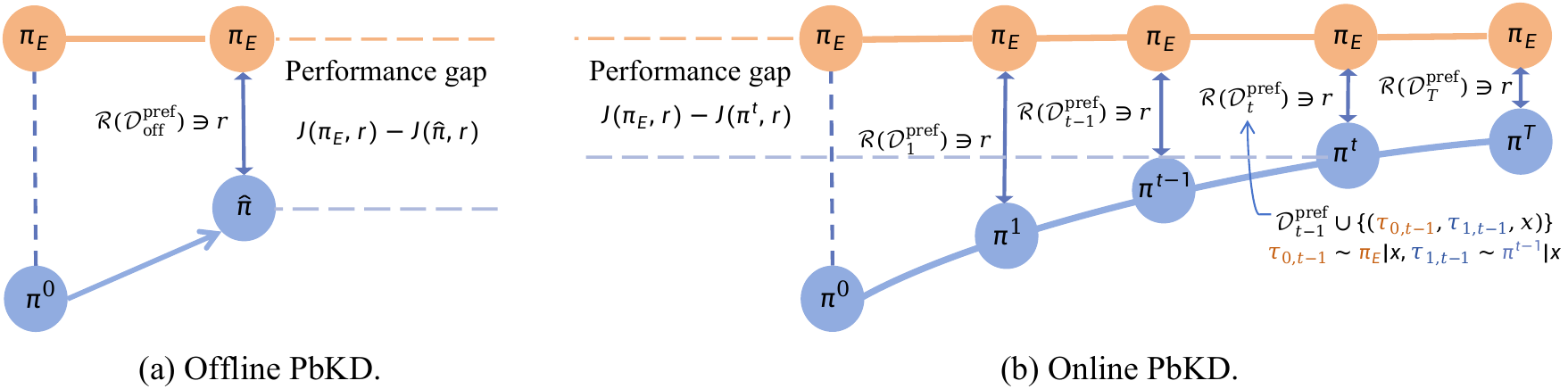}
	\caption{We formulate reward-guided imitation learning as the optimization of the performance gap between the teacher policy $\pi_E$ and the student policy $\hat{\pi}$, defined as $\hat{\pi} := \mathop{\arg\min}_\pi \max_{r} J(\pi_E, r) - J(\pi, r)$ (Eq. (\ref{eq:obj})). In the offline PbKD setting, a pre-collected preference dataset is used to construct a confidence set for constraining the reward model with MLE, i.e., $r \in \mathcal{R}(\mathcal{D}^{\rm pref}_{\rm off})$ (Algorithm \ref{alg:offlinekd}). In the online PbKD setting, new preference data are iteratively collected from the current student policy and incorporated into the confidence set, resulting in time-dependent constraints: $r \in \mathcal{R}(\mathcal{D}^{\rm pref}_t)$ at each iteration $t \in \{1, 2, \ldots, T\}$ (Algorithm \ref{alg:onlinekd}).}
	\label{fig:overall}
\end{figure}
%

\section{Related Work}
\noindent\textbf{Black-Box Knowledge Distillation.}
Black-box knowledge distillation (KD) has emerged as a practical solution for compressing large language models (LLMs) like GPT-4 \cite{brown2024gpt}, Claude 3 \cite{anthropic2024claude3}, and Gemini (Google, 2023) into smaller, more efficient student models. Black-box KD leverages only the output responses of proprietary APIs. Recent efforts such as Alpaca \cite{taori2023alpaca}, Vicuna \cite{chiang2023vicuna}, and Orca \cite{mukherjee2023orca} have explored using these responses as supervision to transfer capabilities like reasoning, tool use, and code generation. Techniques include prompting LLMs to generate Chain-of-Thought (CoT) explanations \cite{hsieh2023distilling,wang2023selfinstruct} rationale-based reasoning \cite{liu2023tinygsm}, or code-language pairs \cite{azerbayev2023explicit}, which are then used to fine-tune smaller models. These approaches demonstrate the potential of black-box LLMs as data generators to train open-source alternatives without requiring access to their internal weights or training data.

\noindent\textbf{White-Box Knowledge Distillation.}
White-box knowledge distillation (KD) builds upon the classical KD framework introduced by Hinton et al. \cite{hinton2015distilling}, where the student model learns from the soft predictions of a larger teacher model. This paradigm has been widely applied to compress large pre-trained language models (PLMs), with early work focusing on logits and intermediate features \cite{jiao2020tinybert,sun2019patient,wang2020minilm}. Subsequent methods extended KD to sequence generation tasks, including machine translation and summarization, by aligning entire output distributions \cite{chen2020distilling,wang2021selective}.
Recent efforts target large language models (LLMs) specifically, exploring divergence-based objectives such as KL, reverse KL (RKL), and Jensen-Shannon (JS) divergence for auto-regressive generation \cite{wen2023f,agarwal2024policy,gu2023minillm}. For instance, Gu et al. \cite{gu2023minillm} introduced a policy gradient method to reduce RKL variance, while Ko et al. \cite{ko2024distillm} proposed an efficient skew-divergence-based technique. Additionally, moment-matching and step-wise supervision have gained attention for better aligning generation dynamics \cite{jiaadversarial}. These works reveal the limitations of traditional divergence objectives and highlight the need for more adaptive and data-efficient strategies in white-box KD for LLMs.

\noindent\textbf{Preference-based Knowledge Distillation.}
While knowledge distillation (KD) has been extensively studied for large language models (LLMs), relatively few efforts have explored the use of preference optimization to guide the distillation process, aiming to align student models with human preferences. Recent works on preference-based knowledge distillation (PbKD) include \cite{li2024direct,zhang2024plad,chen2024knowledgedistillationblackboxlarge}, where a preference loss is incorporated—typically using pairwise comparisons between teacher and student outputs—alongside the conventional maximum likelihood estimation (MLE) loss used in behavior cloning. Notably, Li et al. \cite{li2024direct} and Chen et al. \cite{chen2024knowledgedistillationblackboxlarge} focus on black-box settings, whereas Zhang et al. \cite{zhang2024plad} propose a method that falls under white-box KD. In contrast to these approaches, we propose a more general PbKD framework grounded in imitation learning, applicable to both black-box and white-box scenarios. Our improvements stem from a more robust min-max optimization framework with an online, iterative training scheme that enhances preference alignment and training stability.

\section{Problem Formulation}
In this section, we formulate sequential knowledge distillation as an imitation learning problem for language generation. Specifically, we consider language generation as a \textit{finite-horizon}, \textit{time-independent} sequence-to-sequence decision-making task. Given an input prompt $x \in \mathcal{X}$, the policy generates an output sequence $y = \{y_1, \ldots, y_H\}$ under an episodic Markov Decision Process (MDP), defined by the tuple $(\mathcal{S}, \mathcal{A}, T, r^*)$. Here, $\mathcal{S} = \{s_h\}_{0 \leq h \leq H}$ denotes the state space consisting of output prefixes $s_h = y_{\leq h}$, and $\mathcal{A} = \{a_h\}_{0 \leq h \leq H-1}$ denotes the action space corresponding to predicted tokens $a_h = y_{h+1}$. We assume a deterministic state-transition function $T$ in sequence generation, where each action deterministically advances the state.

Given a policy $\pi$, the prediction of the next action at a state $s$ is modeled as a probability distribution $\pi(\cdot \mid s)$ over $\mathcal{A}$. Starting from an initial state (i.e., prompt) $x \sim d_0$, a trajectory is generated as $(a_0, s_1, a_1, \ldots, s_{H-1}, a_{H-1}, s_H) \sim \pi \mid x$. The value function of policy $\pi$ with respect to a RM $r$ is defined as the expected cumulative reward over the episode: $V^{\pi}(x) = \mathbb{E}_{\tau \sim \pi \mid x} \left[ \sum_{h=0}^{H-1} \gamma^h r(s_h, a_h) \right]$,
where $\gamma \in [0, 1]$ is a discount factor. The overall performance of the policy $\pi$ under RM $r$ is then given by the expected value over the input distribution $d_0$:
$J(\pi, r) = \mathbb{E}_{x \sim d_0} \left[ V^{\pi}(x) \right]$. 

\noindent\textbf{Preference Learning.} 
Given a RM $r$, the preference probability over a binary feedback $o \in \{0,1\}$ for a trajectory pair $(\tau_0, \tau_1)$ conditioned on an input $x$ is modeled using the Bradley-Terry-Luce (BTL) formulation:
\begin{align}
	\mathbb{P}_r(\tau_0 \succ \tau_1 \mid x) = \mathbb{P}_r(o = 1 \mid x, \tau_0, \tau_1) = \sigma\left( r(x, \tau_0) - r(x, \tau_1) \right),
\end{align}
where $r(x, \tau) = \sum_{h=0}^{H-1} \gamma^h r(s_h, a_h)$ denotes the cumulative discounted reward of the trajectory $\tau = (s_1, a_1, \ldots, s_{H-1}, a_{H-1}, s_H)$ conditioned on input $x$, and $\sigma(\cdot)$ is the sigmoid function $\sigma(x) = 1/(1 + \exp(-x))$.

We assume the existence of a ground-truth reward model $r^*$ that induces the true preference distribution, i.e.,  $\mathbb{P}(\tau_0 \succ \tau_1 \mid x) = \mathbb{P}_{r^*}(o = 1 \mid x, \tau_0, \tau_1)$.$\mathbb{P}(\tau_0 \succ \tau_1 \mid x) = \mathbb{P}_{r^*}(o = 1 \mid x, \tau_0, \tau_1)$.
To evaluate the discrepancy between a candidate RM $r \in \mathcal{G}_r$ and the true model $r^*$, we define the maximum likelihood estimation (MLE) loss for a single preference instance as:
\begin{align}
	\ell_r(o; x, \tau_0, \tau_1) = 
	\begin{cases}
		\frac{1}{2} \log \frac{\mathbb{P}_r(o \mid x, \tau_0, \tau_1)}{\mathbb{P}_{r^*}(o \mid x, \tau_0, \tau_1)}, & \text{if } \mathbb{P}_{r^*}(o \mid x, \tau_0, \tau_1) \neq 0 \\
		0, & \text{otherwise}
	\end{cases}
\end{align}

The empirical MLE objective over a dataset of $N$ preference-labeled samples $\mathcal{D}^{\text{pref}} = \{(o_n; x_n, \tau_{0,n}, \tau_{1,n})\}_{n=1}^N$ is given by:
$L_r(\mathcal{D}^{\text{pref}}) = \sum_{n=1}^N \ell_r(o_n; x_n, \tau_{0,n}, \tau_{1,n})$.
We denote by $\mathcal{L}_r = \{\ell_r : r \in \mathcal{G}_r\}$ the class of MLE objective functions corresponding to the reward model function class $\mathcal{G}_r$.

\begin{definition}[\bf $\epsilon$-Bracketing Number of the MLE Class $\mathcal{L}_r$] \label{def:bracketnum}
	Let $\mathcal{L}_r$ be the class of MLE objective functions. A pair $(\ell_1, \ell_2)$ with $\ell_1, \ell_2 \in \mathcal{L}_r$ is called an \emph{$\epsilon$-bracket} if it satisfies:
	\begin{itemize}
		\item $\ell_1(o; x, \tau_0, \tau_1) \leq \ell_2(o; x, \tau_0, \tau_1)$ for all preference samples $(o; x, \tau_0, \tau_1)$;
		\item The $L_1$-distance between the two functions is bounded as $\left\| \ell_1 - \ell_2 \right\|_1 \leq \epsilon$.
	\end{itemize}
	The \emph{$\epsilon$-bracketing number} of $\mathcal{L}_r$ with respect to the $L_\infty$ norm, denoted by $\mathcal{N}_{[]}(\epsilon, \mathcal{L}_r, L_{\infty})$, is the minimal number of such brackets $\{(\ell_{1,n}, \ell_{2,n})\}_{n=1}^N$ required to cover $\mathcal{L}_r$, in the sense that for every $\ell \in \mathcal{L}_r$, there exists an index $n \in [N]$ such that $\ell_{1,n}(o; x, \tau_0, \tau_1) \leq \ell(o; x, \tau_0, \tau_1) \leq \ell_{2,n}(o; x, \tau_0, \tau_1), \quad \forall (o; x, \tau_0, \tau_1)$.
\end{definition}

In the following PAC analysis, we assume that the reward function class $\mathcal{G}_r$ satisfies three conditions: realizability, linearity, and boundedness.

\begin{assumption}[\bf Realizability] \label{asp:rea}
	We assume the reward class is realizable, i.e., the ground-truth reward function $r^*$ lies in the function class: $r^* \in \mathcal{G}_r$.
\end{assumption}

\begin{assumption}[\bf Linearity and Boundedness] \label{asp:linearbound}
	We assume the reward function is linearly parameterized as $r_{\theta}(x, \tau) = \theta^\top \phi(x,\tau)$, $\phi: (x, \tau) \mapsto \phi(x,\tau) \in \mathbb{R}^d$ is a fixed feature extractor. Furthermore, we assume the features and parameters are bounded such that $\Vert \phi(x, \tau) \Vert_{2} \leq 1$ for all $(x, \tau)$ and $\| \theta \|_2 \leq B$ for some constant $B > 0$.
\end{assumption}

\begin{proposition}[\bf Bracketing Number of the MLE Loss Class \cite{zhanprovable}] \label{prop:bracketing}
	Under Assumption~\ref{asp:linearbound}, for any $\epsilon \leq 1$, the $\epsilon$-bracketing number of the MLE objective function class $\mathcal{L}_r$ with respect to the $L_\infty$ norm satisfies: $\log \mathcal{N}_{[]}(\epsilon, \mathcal{L}_r, L_{\infty}) \leq \mathcal{O} \left( d \log \frac{B}{\epsilon} \right)$.
\end{proposition}

\noindent\textbf{Reward-Guided Imitation Learning.}  
We consider an imitation learning formulation for knowledge distillation, where the goal is to learn a student policy $\pi \in \Pi$ that closely imitates an expert (i.e., teacher) policy $\pi_E$. Formally, we aim to minimize the worst-case performance gap between the student and expert policies over a set of plausible reward models inferred from preference data:
\begin{align} \label{eq:obj}
	\hat{\pi} = \mathop{\arg\min}_{\pi \in \Pi} \max_{r \in \mathcal{R}(\mathcal{D}^{\rm pref})} J(\pi_E, r) - J(\pi, r),  
\end{align}
where the reward confidence set $\mathcal{R}(\mathcal{D}^{\rm pref}) \subseteq \mathcal{G}_r$ is constructed via the maximum likelihood estimation (MLE) on the preference dataset $\mathcal{D}^{\rm pref}$:
\begin{align}
	\mathcal{R}(\mathcal{D}^{\rm pref})= \left\{ r \in \mathcal{G}_r:  L_r(\mathcal{D}^{\rm pref}) \geq \max_{r \in \mathcal{G}_r} L_r(\mathcal{D}^{\rm pref}) - \zeta  \right\} 
\end{align}
with $\zeta \geq 0$ being a slack parameter controlling the confidence radius.

The resulting objective in Eq.~\eqref{eq:obj} constitutes a min-max optimization problem between the student policy and an uncertain reward function. The inner maximization identifies the least favorable reward model $r \in \mathcal{R}(\mathcal{D}^{\rm pref})$ that maximizes the performance gap, thereby encouraging robustness to reward uncertainty. Rather than simply using the MLE estimate $\hat{r} = \arg\max_{r \in \mathcal{G}_r} L_r(\mathcal{D}^{\rm pref})$, this approach penalizes student policies that do not generalize well under plausible alternative reward functions, leading to a distributionally robust formulation of imitation learning.

In what follows, we introduce different instantiations of this reward-guided imitation learning framework for knowledge distillation, categorized into three settings:
\begin{enumerate}
 \item Offline preference-based knowledge distillation \textbf{(offline PbKD; Section~\ref{sec:offlinekd})}, which uses a fixed offline preference dataset;
\item Online preference-based knowledge distillation \textbf{(online PbKD; Section~\ref{sec:onlinekd})}, where preferences are collected and incorporated during training;
\item  \textbf{Moment-matching PbKD (Section~\ref{sec:mmpbkd})}, which leverages the preference difference lemma (PDL) and reformulates the objective in terms of $Q$-function matching.
\end{enumerate}
All theoretical results are supported with detailed proofs provided in the Appendix.

\section{Offline Preference-based Knowledge Distillation} \label{sec:offlinekd}

\subsection{Algorithm}

In the offline PbKD setting, we are given a fixed offline preference dataset $\mathcal{D}_{\rm off}^{\rm pref} = \{ (o_n; x_n, \tau_{0,n}, \tau_{1,n})\}_{n=1}^N$, where each input $x_n$ is sampled i.i.d. from a distribution $d_0$, and the corresponding trajectory pair $(\tau_{0,n}, \tau_{1,n})$ is independently sampled from two policies $\mu_0(\cdot \mid x_n)$ and $\mu_1(\cdot \mid x_n)$, respectively. The preference label $o_n \in \{0,1\}$ is sampled according to the ground-truth preference distribution $\mathbb{P}_{r^*}(o_n \mid x_n, \tau_{0,n}, \tau_{1,n})$, which is typically derived from human annotations \cite{christiano2017deep,ziegler2019fine,stiennon2020learning,ouyang2022training} or high-quality AI feedback \cite{guo2024direct,pang2024iterative,cen2024value,cui2024ultrafeedback}.

We present the algorithm for offline preference-based knowledge distillation (offline PbKD) in Algorithm~\ref{alg:offlinekd}. Given the offline preference dataset $\mathcal{D}_{\rm off}^{\rm pref}$, our goal is to obtain the student policy $\hat{\pi}$ that minimizes the worst-case performance gap with respect to the expert policy $\pi_E$, following the reward-guided imitation learning objective in Eq.~(\ref{eq:obj}).

To avoid directly solving the constrained optimization problem involving the reward confidence set, we adopt a Lagrangian relaxation approach. By introducing a Lagrange multiplier $\beta \geq 0$, we convert the original constrained problem into the following unconstrained bi-level optimization formulation:
\begin{align}
	\hat{\pi} = \mathop{\arg\min}_{\pi \in \Pi} \max_{r \in \mathcal{G}_r}  J(\pi_E, r) - J(\pi, r) + \beta L_r(\mathcal{D}_{\rm off}^{\rm pref}) ,
\end{align}
where the reward model $r$ is simultaneously optimized to maximize the likelihood of the observed preferences via the MLE objective $L_r(\mathcal{D}_{\rm off}^{\rm pref})$, while inducing the largest performance gap between the student and expert. This encourages the student policy to be robust against plausible reward functions supported by the offline preference data.

\RestyleAlgo{ruled}
\SetKwComment{Comment}{/* }{ */}

\begin{algorithm}[t!] 
	\caption{Offline PbKD} \label{alg:offlinekd}
	\KwIn{Offline preference dataset $\mathcal{D}_{\rm off}^{\rm pref} = \{ (0_n; x_n, \tau_{0,n}, \tau_{1,n}) \}_{n=1}^N$, student policy class $\Pi$, teacher policy $\pi_E$,  input distribution $d_0$.}
	Compute the best policy under constraints on offline confidence set:
	\begin{align} \label{eq:objoffline}
	\hat{\pi} &= \mathop{\arg\min}_{\pi \in \Pi} \max_{r \in \mathcal{R}(\mathcal{D}^{\rm pref}_{\rm off})} J(\pi_E, r) - J(\pi, r) \\
	{s.t.}\ \mathcal{R}(\mathcal{D}^{\rm pref}_{\rm off}) &= \left\{ r \in \mathcal{G}_r:  L_r(\mathcal{D}^{\rm pref}_{\rm off}) \geq \max_{r \in \mathcal{G}_r} L_r(\mathcal{D}^{\rm pref}_{\rm off}) - \zeta  \right\} 
	\end{align} \\
	\KwOut{The optimized student policy $\hat{\pi}$.}
\end{algorithm}

\subsection{Learning Guarantee}
We evaluate the quality of the learned student policy $\hat{\pi}$ by its suboptimality under the ground-truth reward model $r^*$. Specifically, we consider the performance gap between $\hat{\pi}$ and the optimal policy $\pi^* = \mathop{\arg\max}_{\pi \in \Pi} J(\pi, r^*)$, defined as: $J(\pi^*, r^*) - J(\hat{\pi}, r^*)$. This quantity characterizes how far the learned policy is from the best achievable performance under the true reward.

The following theorem provides a learning guarantee for offline PbKD, bounding the suboptimality of $\hat{\pi}$ in terms of the quantity of offline preference data and the functional complexity.

\begin{definition}[\bf Concentrability Coefficient for Offline Preference] \label{def:concencoefoffline}
	The concentrability coefficient w.r.t. a RM class $\mathcal{G}_r$, the teacher policy $\pi_E$ and the optimal student policy $\pi^*$ is defined as:
	\begin{align}
		C_r(\mathcal{G}_r, \pi_E, \pi^*) = \sup_{r \in \mathcal{G}_r} \frac{\mathbb{E}_{x \sim d_0, \tau_0 \sim \pi_E \mid x, \tau_1 \sim \pi^* \mid x}\left[ r^*(\tau_0) - r^*(\tau_1) - \left( r(\tau_0) - r(\tau_1) \right)\right]}{\sqrt{\mathbb{E}_{x \sim d_0, \tau_0 \sim \mu_0 \mid x, \tau_1 \sim \mu_1 \mid x} \left[ \vert r^*(\tau_0) - r^*(\tau_1) - \left( r(\tau_0) - r(\tau_1) \right) \vert^2 \right]}}
	\end{align}
\end{definition}

The concentrability coefficient, which has also been adopted in the literature \cite{chen2019information,songhybrid,zhanprovable,ye2024online}, is defined as an information-theoretic ratio that quantifies the mismatch between the distribution of the preference data and the distributions induced by the teacher policy $\pi_E$ and the optimal student policy $\pi^*$. It captures this discrepancy in terms of the reward differences between positive and negative preference-labeled outputs. Based on the definition of concentrability coefficient, we derive a suboptimality bound as follows.

\begin{theorem}[\bf Suboptimality]\label{thm:subopt}
	For any $\delta \in (0,1]$, given an offline preference dataset with a size of $N$, and let $\zeta = \mathcal{O}\left( \sqrt{\log(\mathcal{N}_{[]}(\epsilon, \mathcal{L}_r, L_\infty)/\delta)/N} \right)$, then under Assumption \ref{asp:rea}, and assuming that $\Vert 1/\sigma(1-\sigma) \Vert_{\infty} \leq \kappa$, for some constant $\kappa \geq  0$,  we have with probability at least $1-\delta$:
\begin{align*}
	J(\pi^*, r^*) - J(\hat{\pi}, r^*) \leq \mathcal{O} \left(\kappa C_r(\mathcal{G}_r, \pi_E, \pi^*) \sqrt{\frac{\log(\mathcal{N}_{[]}(1/N, \mathcal{L}_r, L_\infty)/\delta)}{N}}  \right)
\end{align*}

If we further assume that Assumption \ref{asp:linearbound} holds, we have
\begin{align*}
	J(\pi^*, r^*) - J(\hat{\pi}, r^*) \leq \mathcal{O} \left( (1+e^{2B})\left\Vert \phi(x, \pi_E) - \phi(x, \pi^*) \right\Vert_{\Sigma^{-1}}\sqrt{\frac{d\log(BN)+\log(1/\delta)}{N}} \right),
\end{align*}
where $\Sigma$ denotes the covariance matrix corresponding to the expected difference between feature vectors of preference pairs, defined as: $\Sigma = \frac{\lambda}{B}I + \mathbb{E}_{x \sim d_0, \tau_0 \sim \mu_0\mid x, \tau_1 \sim \mu_1\mid x} \left[ \left( \phi(x, \tau_0) - \phi(x, \tau_1) \right) \left( \phi(x, \tau_0) - \phi(x, \tau_1) \right)^\top \right]$ given a positive constant $0 < \lambda \leq \mathbb{E}_{x \sim d_0, \tau_0 \sim \mu_0\mid x, \tau_1 \sim \mu_1\mid x} \left[ \vert r^*(\tau_0) - r^*(\tau_1) - \left( r(\tau_0) - r(\tau_1) \right) \vert^2 \right]$.
\end{theorem}

\begin{remark}[\bf Convergence Rate and Sample Complexity]
	Theorem~\ref{thm:subopt} establishes that Algorithm~\ref{alg:offlinekd} achieves a convergence rate of $\mathcal{O}(\sqrt{\log(N/\delta)/N})$ for the suboptimality gap of the student policy, where $N$ is the number of preference samples, and $\delta$ is the confidence level. This implies a sample complexity of $N \geq \tilde{\mathcal{O}}\left( \frac{\log(1/\delta)}{\epsilon^2} \right)$ to guarantee an $\epsilon$-optimal solution, i.e., $J(\pi^*, r^*) - J(\hat{\pi}, r^*) \leq \epsilon$.
\end{remark}

\section{Online Preference-based Knowledge Distillation} \label{sec:onlinekd}
\subsection{Algorithm}
In contrast to the offline setting, the online paradigm allows for leveraging preference data from both the student and teacher policies throughout the training process. Given that obtaining preference annotations from humans or LLMs can be expensive and time-consuming, we make a simplifying assumption: responses from the teacher policy are always preferred over those from the student policy.

We present the online preference-based knowledge distillation \textbf{(online PbKD)} algorithm in Algorithm~\ref{alg:onlinekd}. At each iteration $t-1 \in \{0, 2, \ldots, T-1\}$, we collect a the preference-labeled sample $(o_{t-1} = 1; x_{t-1}, \tau_{0,t-1}, \tau_{1,t-1})$, where $x_{t-1} \sim d_0$ is drawn from the input distribution, and the trajectory pair $(\tau_{0,t-1}, \tau_{1,t-1})$ is generated by sampling from the teacher policy $\pi_E(\cdot \mid x_{t-1})$ and the current student policy $\pi^{t-1}(\cdot \mid x_t)$, respectively.

The reward model is optimized within a confidence set constructed from the preference data collected up to iteration $t$, denoted as $\mathcal{D}_t^{\rm pref} = \{ (o_s = 1; x_s, \tau_{0,s}, \tau_{1,s}) \}_{s=0}^{t-1}$. Analogous to the offline formulation, we introduce a Lagrangian multiplier $\beta$ to convert the constrained optimization problem into an unconstrained bi-level problem regularized by the MLE objective:
\begin{align} \label{eq:blackobjpra}
	\pi^t = \mathop{\arg\min}_{\pi \in \Pi} \max_{r \in \mathcal{G}_r} J(\pi_E, r) - J(\pi, r) + \beta L_{r}(\mathcal{D}_t^{\rm pref})
\end{align}

\begin{algorithm}[t!] 
	\caption{Online PbKD} \label{alg:onlinekd}
	\KwIn{(Empty) set $\mathcal{D}^{\rm pref}_0$, student policy class $\Pi$, teacher policy $\pi_E$, input distribution $d_0$.}
	\For{$t = 1,2, \ldots, T$}{
		Sample an input $x_{t-1} \sim d_0$ and trajectories $\tau_{0,t-1} \sim \pi_{E} \mid x_{t-1}$ and $\tau_{1,t-1} \sim \pi^{t-1} \mid x_{t-1}$ \\
		Add the sample into the previous preference set: \\
		$\mathcal{D}^{\rm pref}_{t} \leftarrow \mathcal{D}^{\rm pref}_{t-1} \cup \{(o_{t-1}=1; x_{t-1}, \tau_{0,t-1}, \tau_{1,t-1})\}$ \\
        Compute the $t$-iteration best policy under constraints on the $t$-iteration confidence set:
        \begin{align}
        	  {\pi}^t &= \mathop{\arg\min}_{\pi \in \Pi} \max_{r \in \mathcal{R}(\mathcal{D}^{\rm pref}_t)} J(\pi_E, r) - J(\pi, r),  \\	
        	  {\rm s.t.}\ \mathcal{R}(\mathcal{D}_t^{\rm pref})&= \left\{ r \in \mathcal{G}_r:  L_r(\mathcal{D}_t^{\rm pref}) \geq \max_{r \in \mathcal{G}_r} L_r(\mathcal{D}_t^{\rm pref}) - \zeta_t  \right\} 
        \end{align}}
	\KwOut{The sequence of optimized student policy $\{\pi^t\}_{t=1}^T$.}
\end{algorithm}

\subsection{Learning Guarantee}
We provide a theoretical guarantee for Algorithm~\ref{alg:onlinekd} by establishing a cumulative regret bound. Specifically, we measure the cumulative performance gap between the sequence of learned student policies and the optimal student policy under the golden reward model $r^*$: $\sum_{t=1}^T J(\pi^*, r^*) - J(\pi^t, r^*)$, where the optimal student policy is defined as $\pi^* = \mathop{\arg\max}_{\pi \in \Pi} J(\pi, r^*)$.

This cumulative regret quantifies how far the iteratively learned policies $\{\pi^t\}_{t=1}^T$ are from the optimal policy in hindsight, and serves as a key metric for evaluating the efficiency and effectiveness of the online preference-based knowledge distillation process.

\begin{definition}[\bf Concentrability Coefficient for Online Preference] \label{def:concoeonline}
	The concentrability coefficient for online preference at iteration $t \in \{1,2, \ldots, T\}$ is defined as:
	\begin{align*}
		C_r^t(\mathcal{G}_r, \pi_E, \pi^*) = \sup_{r \in \mathcal{G}_r} \frac{\mathbb{E}_{x \sim d, \tau_0 \sim \pi_E\mid x, \tau_1 \sim \pi^*\mid x}\left[ r^*(\tau_0) - r^*(\tau_1) - \left( r(\tau_0) - r(\tau_1) \right)\right]}{\sqrt{\sum_{s=0}^{t-1} \mathbb{E}_{x_s \sim d_0, \tau_{0,s} \sim \pi_E\mid x, \tau_{1,s} \sim \pi^s \mid x} \left[ \vert  r^*(\tau_0) - r^*(\tau_1) - \left( r(\tau_0) - r(\tau_1) \right) \vert^2 \right]}}
	\end{align*}
\end{definition}
\begin{theorem}[\bf Regret] \label{thm:regret}
	For any $\delta \in (0,1]$, let $\zeta_t = \mathcal{O}\left( \log(\mathcal{N}_{[]}(1/t, \mathcal{L}_r, L_\infty)/\delta) \right)$ for any $t \in \{1,2,\ldots,T\}$ with a maximum iteration number of $T$, then under Assumption \ref{asp:rea}, we have with probability at least $1-\delta$:
	\begin{align*}
	\sum_{t=1}^T J(\pi^*, r^*) - J(\pi^t, r^*) \leq \mathcal{O} \left( \kappa \sum_{t=1}^T C_r^t(\mathcal{G}_r, \pi_E, \pi^*)  \log \frac{\mathcal{N}_{[]}(1/t, \mathcal{L}_r, L_\infty)}{\delta} \right)
	\end{align*}
	
	If we further assume that Assumption \ref{asp:linearbound} holds, we have
	\begin{align*}
		\sum_{t=1}^T J(\pi^*, r^*) - J(\pi^t, r^*) \leq \mathcal{O} \left( (1+e^{2B}) \sqrt{d \log(1+TB/d\lambda_{1}) \sum_{t=1}^T \left( d\log(Bt) + \log(1/\delta) \right)} \right),
	\end{align*}
	where $0 < \lambda_1 \leq \mathbb{E}_{x \sim d_0, \tau_{0} \sim \pi_E\mid x, \tau_{1} \sim \pi^0\mid x} \left[ \vert  r^*(\tau_0) - r^*(\tau_1) - \left( r(\tau_0) - r(\tau_1) \right) \vert^2 \right]$ denotes a positive constant.
\end{theorem}

\begin{remark}[\bf Convergence Rate and Complexity]
	Theorem \ref{thm:regret} shows that Algorithm \ref{alg:onlinekd} achieves a regret bound in the order of approximately $\mathcal{O}( \sqrt{ T \log T \log(T/\delta) })$ in terms of the number of iterations $T$ and the confidence factor $\delta$. This implies that the average regret, defined as $J(\pi^*, r^*) - \frac{1}{T} \sum_{t=1}^T J(\pi^t, r^*)$, decays at a rate of approximately $\mathcal{O}( \sqrt{ \log T \log(T/\delta) /T})$.  To achieve an $\epsilon$-approximate convergence in average regret, i.e., $\frac{1}{T} \sum_{t=1}^T J(\pi^t, r^*) \geq J(\pi^*, r^*) - \epsilon$
	the number of iterations $T$ must satisfy $T \geq \tilde{\mathcal{O}}\left( \frac{\log(1/\delta)}{\epsilon^2} \right)$. This establishes the iteration complexity of the online PbKD algorithm for achieving $\epsilon$-optimality in regret.
\end{remark}

\section{Moment-Matching PbKD} \label{sec:mmpbkd}

The preference-based knowledge distillation (PbKD) framework proposed in the previous sections can be categorized as \emph{black-box} knowledge distillation, since it relies solely on the trajectories output by the teacher policy without accessing the token-level probabilities computed by the teacher policy, as commonly done in \emph{white-box} distillation approaches~\cite{hinton2015distilling,wen2023f,agarwal2024policy,gu2023minillm}. 

In this section, we extend the PbKD framework to the white-box setting by leveraging the \emph{performance difference lemma} (PDL) to reformulate the optimization in Eq.~(\ref{eq:obj}) as a dual problem. This leads to our proposed approach: moment-matching preference-based knowledge distillation \textbf{(MM PbKD)}.

\begin{proposition}[\bf Performance Difference Lemma~\cite{swamy2021moments,jiaadversarial}] \label{pro:pdl}
	Let $Q^{\pi_E}_r$ denote the $Q$-function under the teacher policy $\pi_E$ and reward function $r \in \mathcal{G}_r$, defined as $Q^{\pi_E}_r(s_h, a_h) = \mathbb{E}_{\tau \sim \pi_E \mid s_h, a_h} \left[ \sum_{h'=h}^{H-1} \gamma^{h'-h} r(s_{h'}, a_{h'}) \right]$. Then, the performance difference between the teacher policy $\pi_E$ and a student policy $\pi$ under reward $r$ satisfies:
	\begin{align*}
		J(\pi_E, r) - J(\pi, r) = \mathbb{E}_{x \sim d_0, \tau \sim \pi \mid x} \left[ \sum_{h=0}^{H-1} \mathbb{E}_{a \sim \pi_E(\cdot \mid s_h)} Q^{\pi_E}_r(s_h, a) - Q^{\pi_E}_r(s_h, a_h) \right]
	\end{align*}
\end{proposition}

We define the class of teacher $Q$-functions induced by $\mathcal{G}_r$ as: $\mathcal{G}_{Q_E} = \left\{ Q^{\pi_E}_r: r \in \mathcal{G}_r \right\}$, where each $f \in \mathcal{G}_{Q_E}$ represents a $Q$-function $Q^{\pi_E}_r$ under some reward model $r$.

Under a deterministic transition assumption, i.e., $T(s_{h+1} \mid s_h, a_h) = 1$, the Bellman equation gives: $Q^{\pi_E}_r(s_h, a_h) = r(s_h, a_h) + \gamma \mathbb{E}_{a' \sim \pi_E(\cdot \mid s_{h+1})} Q^{\pi_E}_r(s_{h+1}, a').$ Thus, the reward function induced by any $f \in \mathcal{G}_{Q_E}$ can be written as: $r(f) := f - \gamma f \circ \pi_E$.

Using Proposition~\ref{pro:pdl}, the original optimization problem in Eq.~(\ref{eq:obj}) can be reformulated as a dual optimization problem:
\begin{align}
	\hat{\pi} = \arg\min_{\pi \in \Pi} \max_{f \in \mathcal{Q}(\mathcal{D}^{\text{pref}})} \mathbb{E}_{x \sim d_0, \tau \sim \pi \mid x} \left[ \sum_{h=0}^{H-1} \mathbb{E}_{a \sim \pi_E(\cdot \mid s_h)} f(s_h, a) - f(s_h, a_h) \right],
\end{align}
where the confidence set $\mathcal{Q}(\mathcal{D}^{\text{pref}}) \subset \mathcal{G}_{Q_E}$ is constructed using the induced reward $r(f)$ as:
\begin{align}
\mathcal{Q}(\mathcal{D}^{\text{pref}}) := \left\{ f \in \mathcal{G}_{Q_E}: L_{r(f)}(\mathcal{D}^{\text{pref}}) \geq \max_{f \in \mathcal{G}_{Q_E}} L_{r(f)}(\mathcal{D}^{\text{pref}}) - \zeta \right\}
\end{align}

This formulation can be interpreted as a \textbf{moment-matching} objective: minimizing the expected discrepancy between (1) the teacher's expected $Q$-value and (2) the student policy’s realized $Q$-value at each step $h \in \{0, 1, \dots, H-1\}$. It naturally lends itself to an on-policy RL algorithm for optimization, such as those used by Swamy et al.~\cite{swamy2021moments}.

Similar to previous sections, we can further instantiate this framework as: \textbf{Offline MM PbKD}, where the confidence set is built using a fixed, pre-collected preference dataset. \textbf{Online MM PbKD}, where preferences are collected dynamically during training. These extensions allow MM PbKD to flexibly adapt to various supervision settings while incorporating white-box information from the teacher.

\section{Experiments}
We evaluate knowledge distillation using both black-box and white-box LLMs.

\subsection{Knowledge Distillation with Black-Box LLMs}

We study knowledge distillation from black-box LLMs, where internal information are inaccessible and interaction is limited to API outputs used to train smaller student models.

\noindent\textbf{Experimental Setup.}  
We choose GPT-4~\cite{achiam2023gpt} as the black-box teacher and initialize the student with LLaMA2 7B~\cite{touvron2023llama}. Training data is sampled from a mix of OpenOrca~\cite{lian2023openorca} and Nectar~\cite{starling2023}, totaling 100,000 prompts. Evaluation uses diverse benchmarks: BBH~\cite{suzgun2023challenging}, AGIEval~\cite{zhong2024agieval}, ARC-Challenge~\cite{clark2018think}, MMLU~\cite{hendryckstest2021}, and GSM8K~\cite{cobbe2021training}.
Offline preference data is pre-collected by fine-tuning several LLMs on 10,000 samples to generate candidate outputs, ranked by GPT-4 feedback. The reward model (RM) is a linear layer atop a fixed Llama-2-7B. For offline PbKD, parameters are randomly initialized; for online PbKD, pretrained on offline preference data. Online PbKD runs for 5 iterations, each sampling 1,000 prompts to generate preference pairs from the current student and teacher, augmenting the dataset for RM training. Results are averaged over three seeds; details in the Appendix.

\noindent\textbf{Main Results.}  
Table~\ref{tbl:black} reports results. Our method surpasses the Vanilla Black-Box KD baseline, which directly fine-tunes on teacher outputs, and improves upon Best-of-$N$ sampling~\cite{stiennon2020learning} (with $N=10$) that selects candidates by RM scores. Compared to Proxy-KD~\cite{chen2024knowledgedistillationblackboxlarge}, which uses a white-box proxy model aligned with the black-box teacher, our preference-optimization based min-max formulation achieves better accuracy and robustness.
Ablation on online PbKD iterations shows consistent gains over offline-only training, with performance steadily improving as iterations increase—highlighting the benefits of iterative reward-guided distillation.

\begin{table*}[t!] 
	\centering
	\caption{Black-box KD on five datasets. $^*$ indicates that the results are from their original papers. We format {\bf the best} and \underline{the second best} results. The results based on Qwen2 are available in Appendix.}
	\label{tbl:black}
	\scalebox{.9}{
		\begin{tabular}{lcccccc}
			\toprule
			{\bf Method}& BBH &AGIEval & ARC-Challenge & MMLU &GSM8K & {\bf Avg.} \\
			\midrule
			Teacher & {\it 88.0} &{\it 56.4} & {\it 93.3} & {\it 86.4} & {\it 92.0} & {\it 83.2} \\ 
			\midrule
			Proxy-KD$^*$ \cite{chen2024knowledgedistillationblackboxlarge} & \underline{53.40} & 36.59 & 71.09 & 51.35 & 53.07 & 53.1 \\ 
		    Vanilla Black-Box KD & 45.6$_{\pm 0.4}$ & 36.3$_{\pm 0.3}$ & 67.2$_{\pm 0.5}$ & 49.5$_{\pm 0.4}$ & 50.0$_{\pm 0.4}$ & 49.7 \\
		    $\llcorner$ Best-of-$N$ & 47.8$_{\pm 0.3}$ & 34.8$_{\pm 0.4}$ & 68.5$_{\pm 0.3}$& 50.2$_{\pm 0.2}$ & 51.7$_{\pm 0.4}$ & 50.6 \\
		    \midrule
			Offline PbKD  &50.7$_{\pm 0.3}$&40.3$_{\pm 0.3}$&70.5$_{\pm 0.6}$&53.9$_{\pm 0.3}$&52.7$_{\pm 0.4}$&53.6 \\ 
			Online PbKD (1-iter) &51.6$_{\pm 0.4}$&42.4$_{\pm 0.3}$&70.4$_{\pm 0.5}$&55.2$_{\pm 0.3}$&54.3$_{\pm 0.6}$&54.8 \\
			Online PbKD (3-iter) &53.0$_{\pm 0.3}$&\textbf{45.2}$_{\pm 0.2}$&\underline{72.5}$_{\pm 0.4}$&\underline{56.0}$_{\pm 0.5}$&\underline{57.4}$_{\pm 0.4}$&56.8 \\
			Online PbKD (5-iter) &\textbf{54.7}$_{\pm 0.6}$&\underline{44.7}$_{\pm 0.3}$&\textbf{73.1}$_{\pm 0.5}$&\textbf{57.3}$_{\pm 0.6}$&\textbf{58.2}$_{\pm 0.4}$&57.6 \\
			\bottomrule
	\end{tabular}}
\end{table*}

\subsection{Knowledge Distillation with White-Box LLMs}

Unlike black-box settings, white-box LLMs provide direct access to token-level probabilities, enabling more effective knowledge extraction via moment distribution matching \cite{gu2023minillm,ko2024distillm,jiaadversarial}.

\noindent\textbf{Experimental Setup.}  
We use a fine-tuned LLaMA2 13B teacher and TinyLLaMA 1.1B student. Training is conducted on 15,000 samples from the {\tt databricks-dolly-15k} dataset \cite{dolly2023introducing}, with 500 samples each for validation and testing. Evaluation benchmarks include DollyEval, SelfInst \cite{wang2022self}, VicunaEval \cite{chiang2023vicuna}, S-NI \cite{wang2022super}, and UnNI \cite{honovich2022unnatural}.
Preference data for PbKD is generated by fine-tuning several LLMs on 5,000 training samples to produce candidate responses, ranked by GPT-4 feedback. The $Q$-function is implemented as a linear layer atop a frozen OpenLLaMA-3B. Online PbKD proceeds in 3 iterations, adding 1,000 new preference samples each time. Results average over three seeds; more details in the Appendix.

\noindent\textbf{Main Results.}  
Table~\ref{tbl:whiteboxmain} shows that white-box KD baselines outperform Black-Box KD, confirming the benefit of token-level probability access. We compare against several distribution matching methods: KL~\cite{hinton2015distilling}, RKL~\cite{lin2020autoregressive}, MiniLLM~\cite{gu2023minillm}, GKD~\cite{agarwal2024policy}, and DistiLLM~\cite{ko2024distillm}. Our method outperforms these by optimizing a stronger implicit imitation objective.
We also benchmark against AMMD~\cite{jiaadversarial}, enhanced with Best-of-$N$ sampling. Unlike AMMD, which shows limited gains from reward guidance alone, our approach integrates reward signals into an online PbKD framework, achieving superior results. Ablations on online iteration numbers align with black-box findings, confirming the value of iterative preference-based distillation.

\begin{table*}[t!] 
	\centering
	\caption{White-box KD on five instruction-following dataset. We format {\bf the best} and \underline{the second best} results. The results based on the Qwen2 series are available in Appendix.}
	\label{tbl:whiteboxmain}
	\scalebox{.8}{
		\begin{tabular}{lcccccccc}
			\toprule
			\multirow{2}{*}{\bf Method}& \multicolumn{2}{c}{{DollyEval}} & \multicolumn{2}{c}{{SelfInst}} & \multicolumn{2}{c}{{VicunaEval}} & S-NI & UnNI \\
			\cmidrule(lr){2-3}\cmidrule(lr){4-5}\cmidrule(lr){6-7}\cmidrule(lr){8-8}\cmidrule(lr){9-9}
			& GPT-4 & R-L & GPT-4 & R-L & GPT-4 & R-L & R-L & R-L \\
			\midrule
			Teacher  & {\it 63.2$_{\pm 0.8}$} & {\it 34.6$_{\pm 0.5}$} & {\it 57.9$_{\pm 0.6}$} & {\it 24.9$_{\pm 0.3}$} &{\it 51.7$_{\pm 0.4}$}  & {\it 24.3$_{\pm 0.3}$} & {\it 41.7$_{\pm 0.5}$} &{\it 41.3$_{\pm 0.5}$}  \\
			\midrule
			Black-Box KD & 42.7$_{\pm 0.5}$ & 24.2$_{\pm 0.4}$ & 41.8$_{\pm 0.4}$& 17.2$_{\pm 0.4}$ & 31.6$_{\pm 0.2}$ & 17.2$_{\pm 0.3}$ & 28.2$_{\pm 0.5}$ & 22.9$_{\pm 0.4}$ \\
			KL \cite{hinton2015distilling} & 47.8$_{\pm 0.7}$ & 25.7$_{\pm 0.3}$ & 45.2$_{\pm 0.2}$ & 16.7$_{\pm 0.5}$ & 36.2$_{\pm 1.2}$ & 18.6$_{\pm 0.3}$ & 27.3$_{\pm 0.4}$ & 26.3$_{\pm 0.2}$ \\
			RKL \cite{lin2020autoregressive} & 52.7$_{\pm 1.0}$ &  24.7$_{\pm 0.4}$ & 44.2$_{\pm 0.6}$ & 17.8$_{\pm 0.3}$ & 40.9$_{\pm 0.7}$ & 18.7$_{\pm 0.1}$ & 32.0$_{\pm 0.8}$ & 27.6$_{\pm 0.4}$ \\
			MiniLLM \cite{gu2023minillm} & 57.5$_{\pm 1.0}$ & 28.0$_{\pm 0.4}$ & 50.4$_{\pm 1.2}$ & 21.0$_{\pm 0.4}$ & 43.8$_{\pm 0.6}$ & 20.4$_{\pm 0.3}$ & 36.8$_{\pm 0.4}$ & 35.9$_{\pm 0.3}$ \\
			GKD \cite{agarwal2024policy} & 56.2$_{\pm 0.7}$ & 28.5$_{\pm 0.5}$ & 51.7$_{\pm 1.1}$ & 20.5$_{\pm 0.6}$ & 44.7$_{\pm 0.8}$ & 19.7$_{\pm 0.4}$ & 35.7$_{\pm 0.2}$ & 32.7$_{\pm 0.5}$ \\
			DistiLLM \cite{ko2024distillm} & 57.6$_{\pm 0.4}$ & 29.8$_{\pm 0.7}$ & 52.7$_{\pm 0.7}$ &  21.7$_{\pm 0.5}$ & 45.7$_{\pm 0.5}$ &19.7$_{\pm 0.6}$ & 37.4$_{\pm 0.6}$ & 36.7$_{\pm 0.4}$\\
			AMMD \cite{jiaadversarial} & 56.7$_{\pm 0.4}$ & 28.5$_{\pm 0.4}$ & 53.6$_{\pm 0.7}$ &20.5$_{\pm 0.7}$ & 46.3$_{\pm 0.4}$ & 20.8$_{\pm 0.6}$ & 38.0$_{\pm 0.6}$ & 37.4$_{\pm 0.2}$ \\
			$\llcorner$ Best-of-$N$ &57.2$_{\pm 0.4}$&27.3$_{\pm 0.3}$&54.2$_{\pm 0.4}$&21.5$_{\pm 0.3}$&46.8$_{\pm 0.4}$&21.0$_{\pm 0.3}$&37.5$_{\pm 0.5}$&37.0$_{\pm 0.4}$ \\ 
			\midrule
			Offline MM PbKD  &58.2$_{\pm 0.5}$&29.0$_{\pm 0.4}$&54.5$_{\pm 0.5}$&22.1$_{\pm 0.2}$&47.5$_{\pm 0.6}$&21.7$_{\pm 0.4}$&38.2$_{\pm 0.6}$&38.3$_{\pm 0.4}$ \\ 
			Online MM PbKD (1-iter) &58.5$_{\pm 0.7}$&29.7$_{\pm 0.5}$&\underline{55.8}$_{\pm 0.4}$&22.7$_{\pm 0.5}$&48.3$_{\pm 0.7}$&22.6$_{\pm 0.7}$&\underline{38.8}$_{\pm 0.4}$&39.1$_{\pm 0.3}$ \\ 
			Online MM PbKD (2-iter) &\underline{59.0}$_{\pm 0.4}$&\underline{30.6}$_{\pm 0.4}$&55.3$_{\pm 0.2}$&\underline{23.4}$_{\pm 0.3}$&\underline{49.2}$_{\pm 0.6}$&\underline{22.8}$_{\pm 0.5}$&38.2$_{\pm 0.7}$&\underline{39.9}$_{\pm 0.2}$ \\ 
			Online MM PbKD (3-iter) &\textbf{59.7}$_{\pm 0.5}$&\textbf{31.3}$_{\pm 0.5}$&\textbf{56.2}$_{\pm 0.4}$&\textbf{23.7}$_{\pm 0.2}$&\textbf{49.8}$_{\pm 0.5}$&\textbf{23.5}$_{\pm 0.4}$&\textbf{39.7}$_{\pm 0.6}$&\textbf{40.2}$_{\pm 0.4}$ \\ 
			\bottomrule
	\end{tabular}}
\end{table*}

\section{Conclusion}
We introduced Preference-based Knowledge Distillation (PbKD), a unified framework that reframes knowledge distillation as a reward-guided imitation learning problem. By directly minimizing the performance gap between student and teacher policies, PbKD provides a more principled and flexible objective than standard behavior cloning. We developed both offline and online variants, with theoretical guarantees on suboptimality and regret. Experiments across black-box and white-box KD settings demonstrate that PbKD consistently outperforms existing methods, highlighting its robustness and broad applicability.

\noindent\textbf{Limitations.}
Despite its advantages, PbKD has two primary limitations: (1) \emph{Preference Labeling}: Obtaining high-quality preference data can be costly or subjective; (2) \emph{Scalability}: Online PbKD introduces computational overhead due to repeated RM training and dataset expansion. Future work may explore automated or self-supervised preference acquisition and tighter integration with reinforcement learning methods.

\bibliographystyle{plain}
\bibliography{neurips_2025}

\appendix


\newpage
	
\section*{Appendix}
\addcontentsline{toc}{section}{Appendix}
	
	\textbf{Table of Contents}
	\vspace{1em}
	\begin{itemize}[leftmargin=2em, label={}]
	\item \textbf{A \quad Supporting Lemmas \dotfill 14}
	\begin{itemize}[leftmargin=4em, label={}]
			\item \hyperlink{proofA1}{A.1 \quad Bounding the Expected $L_1$-Norm of Reward Difference} \dotfill 14
			\item \hyperlink{proofA2}{A.2 \quad Uniform Convergence for IID Preferences} \dotfill 16
			\item \hyperlink{proofA3}{A.3 \quad Uniform Convergence for Martingale Preferences} \dotfill 18
	\end{itemize}
  \end{itemize}
 
 	\begin{itemize}[leftmargin=2em, label={}]
 	\item  \hyperlink{proofB}{\bf B \quad Proofs for Offline PbKD \dotfill 19}
    \end{itemize}
    
    \begin{itemize}[leftmargin=2em, label={}]
    	\item  \hyperlink{proofC}{\bf C \quad Proofs for Online PbKD \dotfill 19}
    \end{itemize}
	
	    \begin{itemize}[leftmargin=2em, label={}]
		\item  \hyperlink{proofD}{\bf D \quad Practical Implementation of Online PbKD \dotfill 23}
    	\end{itemize}
	
		\begin{itemize}[leftmargin=2em, label={}]
		\item \textbf{E \quad Detailed Experimental Setup and Additional Experiments \dotfill 23}
		\begin{itemize}[leftmargin=4em, label={}]
			\item \hyperlink{proofE1}{E.1 \quad Experimental Setup for Black-box Knowledge Distillation} \dotfill 23
			\item \hyperlink{proofE2}{E.2 \quad Experimental Setup for White-box Knowledge Distillation} \dotfill 25
			\item \hyperlink{proofE3}{E.3 \quad Additional Results Based on Qwen2} \dotfill 26
			\item \hyperlink{proofE4}{E.4 \quad Analysis Experiments} \dotfill 27
		\end{itemize}
	\end{itemize}

	\section*{A Supporting Lemmas}
	\hypertarget{proofA1}{}
	\subsection*{A.1 Bounding the Expected $L_1$-Norm of Reward Difference}
	\begin{lemma}[\bf Bounding $L_1$-Norm of Reward Difference by Total Variation Distance] \label{lemma:l1tv}
        Assume that $\Vert 1/\mathbb{P}_r(1 | x, \tau_0, \tau_1)(1-\mathbb{P}_r(1 | x, \tau_0, \tau_1)) \Vert \leq \kappa$ holds for any $r \in 
		\mathcal{G}_r$ and any $(x, \tau_0, \tau_1)$, we have for any trajectory generation distributions $\pi_0$ and $\pi_1$:
		\begin{align}
		\begin{split} \label{eq:lemma1eq1}
		&\mathbb{E}_{x \sim d_0, \tau_0 \sim \pi_0 \mid x, \tau_1 \sim \pi_1 \mid x}\left[ \vert r(x, \tau_0) - r(x, \tau_1) - \left( r^*(x, \tau_0) - r^*(x, \tau_1) \right) \vert^2 \right] \\
		\leq &\kappa^2 \mathbb{E}_{x \sim d_0, \tau_0 \sim \pi_0 \mid x, \tau_1 \sim \pi_1 \mid x} \left[ \left\Vert  \mathbb{P}_{r}(\cdot \mid x, \tau_{0}, \tau_{1}) - \mathbb{P}_{r^*}(\cdot \mid x, \tau_{0}, \tau_{1})  \right\Vert_{\rm TV}^2 \right]
		\end{split}
		\end{align}
		
		If the Assumption \ref{asp:linearbound} further holds, then we have
		\begin{align}
			\begin{split} \label{eq:lemma1eq2}
				&\mathbb{E}_{x \sim d_0, \tau_0 \sim \pi_0, \tau_1 \sim \pi_1}\left[ \vert r(x, \tau_0) - r(x, \tau_1) - \left( r^*(x, \tau_0) - r^*(x, \tau_1) \right) \vert^2 \right] \\
				\leq &c (1 + e^{2B})^2 \mathbb{E}_{x \sim d_0, \tau_0 \sim \pi_0 \mid x, \tau_1 \sim \pi_1 \mid x} \left[ \left\Vert  \mathbb{P}_{r}(\cdot \mid x, \tau_{0}, \tau_{1}) - \mathbb{P}_{r^*}(\cdot \mid x, \tau_{0}, \tau_{1})  \right\Vert_{\rm TV}^2 \right],
			\end{split}
		\end{align}
		where $c$ denotes an universal constant.
	\end{lemma}
	 
	\begin{proof}
		Given any reward function $r \in \mathcal{G}_r$, we formulate the preference probability as: $\mathbb{P}_r(o=1 \mid x, \tau_0, \tau_1) = \sigma(r(x, \tau_0) - r(x, \tau_1))$, where $\sigma$ denotes the \texttt{sigmoid} function, i.e., $\sigma(x) = \frac{1}{1+\exp(-x)}$. By the reverse function of \texttt{sigmoid}: $\sigma^{-1}$, we have for any $(x, \tau_0, \tau_1)$:
		\begin{align*}
		r(x, \tau_0) - r(x, \tau_1) = \sigma^{-1}(\mathbb{P}_r(o=1 \mid x, \tau_0, \tau_1)) = \log \frac{\mathbb{P}_r(o=1 \mid x, \tau_0, \tau_1)}{1-\mathbb{P}_r(o=1 \mid x, \tau_0, \tau_1)}
		\end{align*}
	    
	     By the first derivative $\frac{d\sigma^{-1}}{d \mathbb{P}_r} = \frac{1}{\mathbb{P}_r(1-\mathbb{P}_r)}$, we can apply the Lagrange mean value theorem with some $r' \in \mathcal{G}_r$, and then using the assumption that  $\Vert 1/\mathbb{P}_r(1 | x, \tau_0, \tau_1)(1-\mathbb{P}_r(1 | x, \tau_0, \tau_1)) \Vert \leq \kappa$ holds for any $r \in 
	    \mathcal{G}_r$ and any $(x, \tau_0, \tau_1)$,
		\begin{align*}
		&\vert r(x, \tau_0) - r(x, \tau_1) - (r^*(x, \tau_0) - r^*(x, \tau_1)) \vert \\
		\leq& \frac{1}{\mathbb{P}_{r'}(o=1 \mid x, \tau_0, \tau_1)(1-\mathbb{P}_{r'}(o=1 \mid x, \tau_0, \tau_1))} \vert  \mathbb{P}_{r}(o=1 \mid x, \tau_0, \tau_1) - \mathbb{P}_{r^*}(o=1 \mid x, \tau_0, \tau_1) \vert \\
		\leq& \kappa \vert  \mathbb{P}_{r}(o=1 \mid x, \tau_0, \tau_1) - \mathbb{P}_{r^*}(o=1 \mid x, \tau_0, \tau_1) \vert
		\end{align*}
		
		Then, using the definition of total variation distance (denoted by $\Vert \cdot - \cdot \Vert_{\rm TV}$), we have 
		\begin{align} \label{eq:eq5}
			\begin{split}
				&\left\vert r^*(x, \tau_0) - r^*(x, \tau_1) - \left( r(x, \tau_0) - r(x, \tau_1) \right) \right\vert \\
				\leq& \frac{1}{2} \kappa \Big( \left\vert  \mathbb{P}_{r}(o=1 \mid x, \tau_0, \tau_1) - \mathbb{P}_{r^*}(o=1 \mid x, \tau_0, \tau_1) \right\vert \\
				& \qquad\quad + \left\vert  \mathbb{P}_{r}(o=0 \mid \tau_0, \tau_1; x) - \mathbb{P}_{r^*}(o=0 \mid x, \tau_0, \tau_1) \right\vert \Big) \\
				=& \kappa  \left\Vert  \mathbb{P}_{r}(\cdot \mid x, \tau_0, \tau_1) - \mathbb{P}_{r^*}(\cdot \mid x, \tau_0, \tau_1)  \right\Vert_{\rm TV}
			\end{split}
		\end{align}
	     
	     Then, taking expectation on both sides, we complete the first inequality of Eq. (\ref{eq:lemma1eq1}). For the second inequality, we have for any $(x, \tau_0, \tau_1)$, and any $\theta$, there exists an universal constant $c$,

		\begin{align*}
		&\frac{1}{\mathbb{P}_{r}(o=1 \mid x, \tau_0,\tau_1)(1-\mathbb{P}_{r}(o=1 \mid x, \tau_0, \tau_1))}  \\
		=& \frac{1}{\sigma(\theta^\top (\phi(x, \tau_0) - \phi(x, \tau_1)))(1- \sigma(\theta^\top (\phi(x, \tau_0) - \phi(x, \tau_1))} \\
		=& \frac{(1+\exp(- \theta^\top(\phi(x, \tau_0) - \phi(x, \tau_1))))^2}{\exp(-\theta^\top(\phi(x, \tau_0) - \phi(x, \tau_1)))} \\
		=& \exp(\theta^\top(\phi(x, \tau_0) - \phi(x, \tau_1))) + \exp(-\theta^\top(\phi(x, \tau_0) - \phi(x, \tau_1))) + 2 \\
		\leq& c (1 + e^{2B}),
		\end{align*}
		where the last step follows from Assumption \ref{asp:linearbound}.
		
		Therefore, Eq. (\ref{eq:lemma1eq2}) can be derived by set $\kappa = c (1 + e^{B})$.
	\end{proof}
	
	\begin{lemma}[\bf Bounding Total Variation Distance by $-\log\mathbb{E}\exp \ell_r$] \label{lemma:boundtv}
		Given an RM $r \in \mathcal{G}_r$ and let 
		\begin{align}
			\ell_r(o; x, \tau_0, \tau_1) = 
			\begin{cases}
				\frac{1}{2} \log \frac{\mathbb{P}_r(o \mid x, \tau_0, \tau_1)}{\mathbb{P}_{r^*}(o \mid x, \tau_0, \tau_1)}, & {\rm if}\ \mathbb{P}_{r^*}(o \mid x, \tau_0, \tau_1) \neq 0 \\
				0, & {\rm otherwise}
			\end{cases}
		\end{align}
		denote the MLE objective w.r.t. any $(x, \tau_0, \tau_1)$. Then, we have for any trajectory generation distributions $\pi_0$ and $\pi_1$:
		\begin{align}
			\begin{split}
				&\mathbb{E}_{x \sim d_0, \tau_0 \sim \pi_0 \mid x, \tau_1 \sim \pi_1 \mid x} \left[ \left\Vert  \mathbb{P}_{r}(\cdot \mid x, \tau_{0}, \tau_{1}) - \mathbb{P}_{r^*}(\cdot \mid x, \tau_{0}, \tau_{1})  \right\Vert_{\rm TV}^2 \right]  \\
				\leq& -2 \log \mathbb{E}_{x \sim d_0, \tau_0 \sim \pi_0 \mid x, \tau_1 \sim \pi_1 \mid x \atop o \sim \mathbb{P}_{r^*}(\cdot \mid x, \tau_0, \tau_1)} \exp\left( \ell_r(o; x, \tau_0, \tau_1) \right)
			\end{split}
		\end{align}
	\end{lemma}
	
	\begin{proof}
		By the Cauchy–Schwarz inequality, we have
		\begin{align*}
			&\left\Vert  \mathbb{P}_{r}(\cdot \mid x, \tau_{0}, \tau_{1}) - \mathbb{P}_{r^*}(\cdot \mid x, \tau_{0}, \tau_{1})  \right\Vert_{\rm TV}^2 \\
			=& \frac{1}{4} \left(  \sum_{o \in \{0,1\}} \lvert    \mathbb{P}_{r}(o \mid x, \tau_{0}, \tau_{1}) - \mathbb{P}_{r^*}(o \mid x, \tau_{0}, \tau_{1})  \rvert   \right)^2 \\
			\leq& \frac{1}{4}  \sum_{o \in \{0,1\}} \lvert  \sqrt{\mathbb{P}_{r}(o \mid x, \tau_{0}, \tau_{1})} - \sqrt{\mathbb{P}_{r^*}(o \mid x, \tau_{0}, \tau_{1})}  \rvert^2 \sum_{o \in \{0,1\}} \lvert  \sqrt{\mathbb{P}_{r}(o \mid x, \tau_{0}, \tau_{1})} + \sqrt{\mathbb{P}_{r^*}(o \mid x, \tau_{0}, \tau_{1})}  \rvert^2  \\
			\leq& \sum_{o \in \{0,1\}} \lvert  \sqrt{\mathbb{P}_{r}(o \mid x, \tau_{0}, \tau_{1})} - \sqrt{\mathbb{P}_{r^*}(o \mid x, \tau_{0}, \tau_{1})}  \rvert^2 
		\end{align*}
		
		Take the expectation on both sides, we have
		\begin{align*}
			&\mathbb{E}_{x \sim d_0, \tau_0 \sim \pi_0 \mid x, \tau_1 \sim \pi_1 \mid x} \left[\left\Vert  \mathbb{P}_{r}(\cdot \mid x, \tau_{0}, \tau_{1}) - \mathbb{P}_{r^*}(\cdot \mid x, \tau_{0}, \tau_{1})  \right\Vert_{\rm TV}^2  \right] \\
			\leq& \mathbb{E}_{x \sim d_0, \tau_0 \sim \pi_0 \mid x, \tau_1 \sim \pi_1 \mid x} \left[   \sum_{o \in \{0,1\}} \lvert  \sqrt{\mathbb{P}_{r}(o \mid x, \tau_{0}, \tau_{1})} - \sqrt{\mathbb{P}_{r^*}(o \mid x, \tau_{0}, \tau_{1})}  \rvert^2    \right]
		\end{align*}
		
		Then, we have 
		\begin{align*}
			&\mathbb{E}_{x \sim d_0, \tau_0 \sim \pi_0 \mid x, \tau_1 \sim \pi_1 \mid x} \left[\sum_{o \in \{0,1\}} \lvert  \sqrt{\mathbb{P}_{r}(o \mid x, \tau_{0}, \tau_{1})} - \sqrt{\mathbb{P}_{r^*}(o \mid x, \tau_{0}, \tau_{1})}  \rvert^2 \right] \\
			=& \mathbb{E}_{x \sim d_0, \tau_0 \sim \pi_0 \mid x, \tau_1 \sim \pi_1 \mid x} \left[ \sum_{o \in \{0,1\}}  \mathbb{P}_{r}(o \mid x, \tau_{0}, \tau_{1}) + \mathbb{P}_{r^*}(o \mid x, \tau_{0}, \tau_{1}) - 2 \sqrt{\mathbb{P}_{r}(o \mid x, \tau_{0}, \tau_{1})  \mathbb{P}_{r^*}(o \mid x, \tau_{0}, \tau_{1})} \right] \\
			=&  2- 2 \mathbb{E}_{x \sim d_0, \tau_0 \sim \pi_0 \mid x, \tau_1 \sim \pi_1 \mid x \atop o \sim \mathbb{P}_{r^*}(o \mid x, \tau_0, \tau_1)} \left[ \sqrt{\frac{\mathbb{P}_{r}(o \mid x, \tau_{0}, \tau_{1}) }{\mathbb{P}_{r^*}(o \mid x, \tau_{0}, \tau_{1}) }}  \right] \\
			\leq& -2 \log \mathbb{E}_{x \sim d_0, \tau_0 \sim \pi_0 \mid x, \tau_1 \sim \pi_1 \mid x \atop o \sim \mathbb{P}_{r^*}(o \mid x, \tau_0, \tau_1)} \left[ \sqrt{\frac{\mathbb{P}_{r}(o \mid x, \tau_{0}, \tau_{1}) }{\mathbb{P}_{r^*}(o \mid x, \tau_{0}, \tau_{1}) }}  \right] \\
			=& -2 \log \mathbb{E}_{x \sim d_0, \tau_0 \sim \pi_0 \mid x, \tau_1 \sim \pi_1 \mid x \atop o \sim \mathbb{P}_{r^*}(o \mid x, \tau_0, \tau_1)} \left[ \exp \left( \frac{1}{2} \log \frac{\mathbb{P}_{r}(o \mid x, \tau_{0}, \tau_{1}) }{\mathbb{P}_{r^*}(o \mid x, \tau_{0}, \tau_{1}) } \right) \right],
		\end{align*}
		where the inequality follows from the fact that $1-x \leq -\log x$.
		
		Combining the above inequalities, we complete the proof of Lemma \ref{lemma:boundtv}.
	\end{proof}

	Directly combining Lemma \ref{lemma:l1tv} and Lemma \ref{lemma:boundtv}, we obtain the following lemma.
 	
	\begin{lemma}[\bf Bounding $L_1$-Norm of Reward Difference by $-\log\mathbb{E}\exp \ell_r$] \label{lemma:l1logeexp}
		Given an RM $r \in \mathcal{G}_r$ and let 
		\begin{align}
			\ell_r(o; x, \tau_0, \tau_1) = 
			\begin{cases}
				\frac{1}{2} \log \frac{\mathbb{P}_r(o \mid x, \tau_0, \tau_1)}{\mathbb{P}_{r^*}(o \mid x, \tau_0, \tau_1)}, & {\rm if}\ \mathbb{P}_{r^*}(o \mid x, \tau_0, \tau_1) \neq 0 \\
				0, & {\rm otherwise}
			\end{cases}
		\end{align}
		denote the MLE objective w.r.t. any $(x, \tau_0, \tau_1)$.  Assume that $\Vert 1/\mathbb{P}_r(1 | x, \tau_0, \tau_1)(1-\mathbb{P}_r(1 | x, \tau_0, \tau_1)) \Vert \leq \kappa$ holds for any $r \in 
		\mathcal{G}_r$ and any $(x, \tau_0, \tau_1)$.  Given some trajectory generation distributions $\pi_0$ and $\pi_1$, we have
		\begin{align}
			\begin{split} \label{eq:lemma1eq1}
				&\mathbb{E}_{x \sim d_0, \tau_0 \sim \pi_0 \mid x, \tau_1 \sim \pi_1 \mid x}\left[ \vert r(x, \tau_0) - r(x, \tau_1) - \left( r^*(x, \tau_0) - r^*(x, \tau_1) \right) \vert^2 \right] \\
				\leq & -2 \kappa^2 \log \mathbb{E}_{x \sim d_0, \tau_0 \sim \pi_0 \mid x, \tau_1 \sim \pi_1 \mid x \atop o \sim \mathbb{P}_{r^*}(\cdot \mid x, \tau_0, \tau_1)} \exp\left( \ell_r(o; x, \tau_0, \tau_1) \right)
			\end{split}
		\end{align}
		
		If Assumption \ref{asp:linearbound} further holds, then we have
		\begin{align}
			\begin{split} \label{eq:lemma1eq2}
				&\mathbb{E}_{x \sim d_0, \tau_0 \sim \pi_0, \tau_1 \sim \pi_1}\left[ \vert r(x, \tau_0) - r(x, \tau_1) - \left( r^*(x, \tau_0) - r^*(x, \tau_1) \right) \vert^2 \right] \\
				\leq &- c (1 + e^{2B})^2  \log \mathbb{E}_{x \sim d_0, \tau_0 \sim \pi_0 \mid x, \tau_1 \sim \pi_1 \mid x \atop o \sim \mathbb{P}_{r^*}(\cdot \mid x, \tau_0, \tau_1)} \exp\left( \ell_r(o; x, \tau_0, \tau_1) \right),
			\end{split}
		\end{align}
		where $c > 0$ denotes an universal constant.
	\end{lemma}
	
	\hypertarget{proofA2}{}
	\subsection*{A.2 Uniform Convergence for IID Preferences}
	\begin{lemma}[\bf IID Exponential Inequality] \label{lemma:iidei}
		Let $\mathcal{D}_N = \{\xi_n\}_{n=1}^N$ be $N$ random variables drawn i.i.d. from a distribution $D$. For any $\delta \in (0,1]$, It holds with probability at least $1-\delta$:
		\begin{align}
			 -\log \mathbb{E}_{\xi \sim D} [\exp(\xi)] \leq - \frac{1}{N} \sum_{n=1}^N \xi_n + \frac{\log(1/\delta)}{N}
		\end{align}
	\end{lemma}
	
	\begin{proof}
		Given $N$ i.i.d. random variables $\mathcal{D}_N = \{\xi_n\}_{n=1}^N$ drawn from a distribution $D$, we can easily obtain:
		\begin{align}
		\mathbb{E}_{\mathcal{D}_N} \exp \left(  \sum_{n=1}^N \xi_n - \sum_{n=1}^N \log \mathbb{E}_{\xi_n \sim D} [e^{\xi_n}] \right) = 1
		\end{align}
		
		By Markov's inequality, we have for any $\epsilon > 0$:
		\begin{align}
			\mathbb{P}_{\mathcal{D}_N}  \left\{   \sum_{n=1}^N \xi_n - \sum_{n=1}^N \log \mathbb{E}_{\xi_n \sim D} [e^{\xi_n}]  \geq \log(1/\delta) \right\} \leq \delta
		\end{align}
		
		Then, we complete the proof of Lemma \ref{lemma:iidei} by using the fact that $\{\xi_n\}_{n=1}^N$ are i.i.d. random variables.
	\end{proof}
	
	\begin{lemma}[\bf Uniform Convergence for IID Preference with Bracketing Number] \label{lemma:uniformiid}
		Denote $\mathcal{L}_r = \{\ell_r : r \in \mathcal{G}_r\}$ as the MLE objective function class. For any $\delta \in (0,1]$, we have with probability at least $1-\delta$,
		\begin{align}
		\forall r \in \mathcal{\mathcal{G}}_r,\ - \log \mathbb{E}_{x \sim d_0, \tau_0 \sim \mu_0 \mid x, \tau_1 \sim \mu_1 \mid x \atop o \sim \mathbb{P}_{r^*}(\cdot \mid x, \tau_0, \tau_1)} \exp\left( \ell_r(o; x, \tau_0, \tau_1) \right) \leq \mathcal{O} \left( \frac{\log\left( \mathcal{N}_{[]}(1/N, \mathcal{L}_r, L_\infty)/ \delta \right)}{N}  \right)
		\end{align}
 	\end{lemma}
	\begin{proof}
		Given a offline preference dataset $\mathcal{D}^{\rm pref}_{\rm off} = \{(o_n; x_n, \tau_{0,n}, \tau_{1,n})\}_{n=1}^N$, We assume that $(o_n; x_n, \tau_{0,n}, \tau_{1,n})$ are i.i.d. preference data, which are sampled by first drawing $x_n \overset{i.i.d.}{\sim} d_0$, and then $\tau_{0,n} \overset{i.i.d.}{\sim} \mu_0 \mid x_n$, $\tau_{1,n} \overset{i.i.d.}{\sim} \mu_1 \mid x_n$, following by $o_n \sim \mathbb{P}_{r^*}(\cdot \mid x_n, \tau_{0,n}, \tau_{1,n})$.  
		
		For simplicity, we denote
		\begin{align*}
			u(\ell; \mathcal{D}^{\rm pref}_{\rm off}) = - \log \!\!\mathop{\mathbb{E}}\limits_{x \sim d_0, \tau_0 \sim \mu_0 \mid x, \tau_1 \sim \mu_1 \mid x \atop o \sim \mathbb{P}_{r^*}(\cdot \mid x, \tau_0, \tau_1)}\!\! \exp\left( \ell(o; x, \tau_0, \tau_1) \right) + \frac{1}{N} \sum_{n=1}^N \ell(o_n; x_n, \tau_{0,n}, \tau_{1,n}),
		\end{align*}
		
		Given a minimal $1/N$-bracket set of MLE objective function class $\widetilde{\mathcal{L}}_r$ with a size of $\mathcal{N}_{[]}(1/N, \mathcal{L}_r, L_\infty)$, we have for any $\ell \in \mathcal{L}_r$, there exist a bracket $[\ell_1, \ell_2] \in \widetilde{\mathcal{L}}_r$, such that
	    \begin{align} \label{eq:lemma5eq1}
	    	\begin{split}
	    	& \sup_{\ell \in\mathcal{L}_r}	u(\ell; \mathcal{D}^{\rm pref}_{\rm off}) \\
	    	=& \sup_{\ell \in\mathcal{L}_r,  [\ell_1, \ell_2] \in \widetilde{\mathcal{L}}_r} \left( u(\ell, \mathcal{D}^{\rm pref}_{\rm off}) -  u(\ell_1, \mathcal{D}^{\rm pref}_{\rm off}) + u(\ell_1, \mathcal{D}^{\rm pref}_{\rm off}) \right) \\
	       \overset{(i)}{\leq}& \sup_{[\ell_1, \ell_2] \in \widetilde{\mathcal{L}}_r} \left( 2\Vert \ell_2 - \ell_1 \Vert_\infty + u(\ell_1, \mathcal{D}^{\rm pref}_{\rm off}) \right) \\
	        \overset{(ii)}{\leq}& \frac{2}{N} + \sup_{[\ell_1, \ell_2] \in \widetilde{\mathcal{L}}_r} u(\ell_1, \mathcal{D}^{\rm pref}_{\rm off}) ,
	        \end{split}
	    \end{align}
		
		where $(i)$ follows from the following derivation and $(ii)$ follows from the definition of $1/N$-bracketing number (refer to Definition \ref{def:bracketnum}).
		\begin{align*}
			&u(\ell, \mathcal{D}^{\rm pref}_{\rm off}) -  u(\ell_1, \mathcal{D}^{\rm pref}_{\rm off}) \\
			=& \frac{1}{N} \sum_{n=1}^N \left( \ell(o_n; x_n, \tau_{0,n}, \tau_{1,n}) - \ell_1(o_n; x_n, \tau_{0,n}, \tau_{1,n}) \right) \\
			&+\log \!\!\mathop{\mathbb{E}}\limits_{x \sim d_0, \tau_0 \sim \mu_0 \mid x, \tau_1 \sim \mu_1 \mid x \atop o \sim \mathbb{P}_{r^*}(\cdot \mid x, \tau_0, \tau_1)}\!\! \exp\left( \ell_1(o; x, \tau_0, \tau_1) \right) -  \log \!\!\mathop{\mathbb{E}}\limits_{x \sim d_0, \tau_0 \sim \mu_0 \mid x, \tau_1 \sim \mu_1 \mid x \atop o \sim \mathbb{P}_{r^*}(\cdot \mid x, \tau_0, \tau_1)}\!\! \exp\left( \ell(o; x, \tau_0, \tau_1) \right) \\
			\leq& 2 \Vert \ell - \ell_1 \Vert_\infty \leq 2 \Vert \ell_2 - \ell_1 \Vert_\infty
		\end{align*}
		
		Using Lemma \ref{lemma:iidei} with the union bound, we have for any $\delta \in (0,1]$, with probability at least $1-\delta$,
		\begin{align*}
			 \sup_{[\ell_1, \ell_2] \in \widetilde{\mathcal{L}}_r} u(\ell_1, \mathcal{D}^{\rm pref}_{\rm off}) \leq \mathcal{O} \left( \frac{\log\left( \mathcal{N}_{[]}(1/N, \mathcal{L}_r, L_\infty)/ \delta \right)}{N}  \right)
		\end{align*}
		
		Then, coming back to Eq. (\ref{eq:lemma5eq1}), we have for any $\delta \in (0,1]$, with probability at least $1-\delta$,
		\begin{align} \label{eq:2222}
			\begin{split}
			\forall \ell \in \mathcal{L}_r,\ &- \log \!\!\mathop{\mathbb{E}}\limits_{x \sim d_0, \tau_0 \sim \mu_0 \mid x, \tau_1 \sim \mu_1 \mid x \atop o \sim \mathbb{P}_{r^*}(\cdot \mid x, \tau_0, \tau_1)}\!\! \exp\left( \ell(o; x, \tau_0, \tau_1) \right) \\
			 \leq& -\frac{1}{N} \sum_{n=1}^N \ell(o_n; x_n, \tau_{0,n}, \tau_{1,n}) + \mathcal{O} \left( \frac{\log\left( \mathcal{N}_{[]}(1/N, \mathcal{L}_r, L_\infty)/ \delta \right)}{N}  \right)
			 \end{split}
		\end{align}
		
        Furthermore, we have 
        \begin{align} \label{eq:3333}
        	\begin{split}
        	-\frac{1}{N} \sum_{n=1}^N \ell(o_n; x_n, \tau_{0,n}, \tau_{1,n}) =& \frac{1}{N}  \sum_{n=1}^N \log \frac{\mathbb{P}_{r^*}(o_n\mid \tau_{0,n}, \tau_{1,n}; x_n)}{\mathbb{P}_{r}(o_n\mid \tau_{0,n}, \tau_{1,n}; x_n)}  \\
        	\leq& \sum_{n=1}^N \log \frac{\mathbb{P}_{\hat{r}}(o_n\mid \tau_{0,n}, \tau_{1,n}; x_n)}{\mathbb{P}_{r}(o_n\mid \tau_{0,n}, \tau_{1,n}; x_n)}  \leq \zeta_r,
        	\end{split}
        \end{align}
		where the inequality follows from the fact that $\hat{r} = \mathop{\arg\max}_{r \in \mathcal{G}_r} \sum_{n=1}^N \log \mathbb{P}_{r} (o_n\mid \tau_{0,n}, \tau_{1,n}; x_n)$. Combining Eq. (\ref{eq:2222}) and Eq. (\ref{eq:3333}) and then set $\zeta_r \leq \mathcal{O}(\log\left( \mathcal{N}_{[]}(1/N, \mathcal{L}_r, L_\infty)/ \delta \right)/N )$, we complete the proof of Lemma \ref{lemma:uniformiid}.
	\end{proof}
	
	\subsection*{A.3 Uniform Convergence for Martingale Preferences}
	\hypertarget{proofA3}{}
	\begin{lemma}[\bf Martingale Exponential Inequality]
		Let $\{ \xi_s \}_{s=0}^\infty$ be a sequence of real-valued random variables adapted to filtration $\{\mathcal{F}_s \}$, It holds with probability $1-\delta$ such that for any $s \geq 1$,
		\begin{align}
			- \sum_{s=0}^t \log \mathbb{E} \left[ \exp \left( \xi_s \right) \mid \mathcal{F}_{s-1}  \right] \leq - \sum_{s=0}^t \xi_s + \log(1/\delta)
		\end{align}
	\end{lemma}
	\begin{proof}
		See Theorem 13.2 of Zhang \cite{zhang2023mathematical} for a detailed proof.
	\end{proof}
	
	\begin{lemma}[\bf Uniform Convergence for Martingale Preference with Bracketing Number] \label{lemma:unformart} 
		Denote $\mathcal{L}_r = \{\ell_r : r \in \mathcal{G}_r\}$ as the MLE objective function class, and $\mathcal{N}_{[]}(1/t, \mathcal{L}_r, L_\infty)$ as its $1/t$-bracketing number of w.r.t. the $L_\infty$-norm. Then for any $\delta \in (0,1]$, we have with probability at least $1-\delta$: $\forall r \in \mathcal{G}_r$,
		\begin{align}
			- \sum_{s=0}^{t-1} \log \mathbb{E}_{x_s \sim d_0, \tau_{0,s} \sim \pi_E \mid x_s, \tau_{1,s} \sim \pi^s \mid x_s \atop o_s \sim \mathbb{P}_{r^*}(\cdot \mid x_s, \tau_{0,s}, \tau_{1,s})} \exp\left(\ell_r(o_s; x_s, \tau_{0,s}, \tau_{1,s}) \right) \leq \mathcal{O} \left( \log \frac{\mathcal{N}_{[]}(1/t, \mathcal{L}_r, L_\infty)}{\delta} \right)
		\end{align}
	\end{lemma}
	\begin{proof}
	
	Applying the above lemma to $\{ \xi_s = \ell_r(o_s; \tau_{0,s}, \tau_{1,s}, x_s) \}_{s=0}^\infty$ along with the filtration $\{ \mathcal{F}_s \}$ with $\mathcal{F}_s$ given by the $\sigma$-algebra of $\{ (o_i; \tau_{0,i}, \tau_{1,i}, x_i) \}_{i=0}^s$, we conclude that given a reward function $r \in \mathcal{G}_r$, it holds with probability $1-\delta$ that
	\begin{align} \label{eq:martei}
		\begin{split}
		&- \sum_{s=0}^{t-1} \log \mathbb{E}_{x_s \sim d^s, \tau_{0,s} \sim \pi_E, \tau_{1,s} \sim \pi^s \atop o_s \sim p_{r^*}(\cdot \mid x_s, \tau_{0,s}, \tau_{1,s})} \exp\left( \ell_r(o_s; x_s, \tau_{0,s}, \tau_{1,s}) \right) \\
		=& - \sum_{s=0}^{t-1} \log \mathbb{E} \left[ \exp\left( \ell_r(o_s; x_s, \tau_{0,s}, \tau_{1,s}) \right) \mid \mathcal{F}_{s-1} \right] \\
		\leq& -  \sum_{s=0}^{t-1} \ell_r(o_s; x_s, \tau_{0,s}, \tau_{1,s}) + \log(1/\delta)
		\end{split}
	\end{align}

    Given a minimal $1/t$-bracket set of MLE objective function class $\widetilde{\mathcal{L}}_r$ with a size of $\mathcal{N}_{[]}(1/t, \mathcal{L}_r, L_\infty)$, we use a similar procedure as used in the proof of Lemma \ref{lemma:iidei},

	Using $1/t$-bracketing number $\mathcal{N}_{[]}(1/t, \mathcal{L}_r, L_\infty)$, we have for any $\delta \in (0,1]$, with probability at least $1 - \delta$,
	\begin{align*}
		& \sup_{\ell \in \mathcal{L}_r} \left\{ \sum_{s=0}^{t-1} \ell(o_s; x_s, \tau_{0,s}, \tau_{1,s}) - \sum_{s=0}^{t-1} \log \mathbb{E}_{x_s \sim d_0, \tau_{0,s} \sim \pi_E \mid x_s, \tau_{1,s} \sim \pi^s \mid x_s \atop o_s \sim \mathbb{P}_{r^*}(\cdot \mid x_s, \tau_{0,s}, \tau_{1,s})} \exp\left( \ell(o_s; x_s, \tau_{0,s}, \tau_{1,s}) \right) \right\} \\
		\leq& 2 + \sup_{\ell \in \widetilde{\mathcal{L}}_r} \left\{ \sum_{s=0}^{t-1} \ell(o_s; x_s, \tau_{0,s}, \tau_{1,s}) - \sum_{s=0}^{t-1} \log \mathbb{E}_{x_s \sim d_0, \tau_{0,s} \sim \pi_E \mid x_s, \tau_{1,s} \sim \pi^s \mid x_s \atop o_s \sim \mathbb{P}_{r^*}(\cdot \mid x_s, \tau_{0,s}, \tau_{1,s})} \exp\left( \ell(o_s; x_s, \tau_{0,s}, \tau_{1,s}) \right)\right\} 
	\end{align*}
	
	Using an union bound by Eq. (\ref{eq:martei}), we have for any $\delta \in (0,1]$, with probability at least $1 - \delta$,
	\begin{align*}
		&\sup_{\ell \in \widetilde{\mathcal{L}}_r} \left\{ \sum_{s=0}^{t-1} \ell(o_s; x_s, \tau_{0,s}, \tau_{1,s}) - \sum_{s=0}^{t-1} \log \mathbb{E}_{x_s \sim d_0, \tau_{0,s} \sim \pi_E \mid x_s, \tau_{1,s} \sim \pi^s \mid x_s \atop o_s \sim \mathbb{P}_{r^*}(\cdot \mid x_s, \tau_{0,s}, \tau_{1,s})} \exp\left( \ell(o_s; x_s, \tau_{0,s}, \tau_{1,s}) \right)\right\}  \\
		\leq& \mathcal{O} \left( \log \frac{\mathcal{N}_{[]}(1/t, \mathcal{L}_r, L_\infty)}{\delta} \right) 
	\end{align*}
	
	Then, we have for any $\delta \in (0,1]$, with with probability at least $1 - \delta$,
	\begin{align*}
		\forall r \in \mathcal{G}_r,  &- \sum_{s=0}^{t-1} \log \mathbb{E}_{x_s \sim d_0, \tau_{0,s} \sim \pi_E \mid x_s, \tau_{1,s} \sim \pi^s \mid x_s \atop o_s \sim \mathbb{P}_{r^*}(\cdot \mid x_s, \tau_{0,s}, \tau_{1,s})} \exp\left( \ell_r(o_s; x_s, \tau_{0,s}, \tau_{1,s}) \right) \\
		\leq&  - \sum_{s=0}^{t-1} \ell_r(o_s; x_s, \tau_{0,s}, \tau_{1,s}) + \mathcal{O} \left( \log \frac{\mathcal{N}_{[]}(1/t, \mathcal{L}_r, L_\infty)}{\delta} \right)  \\
		\leq& - \sum_{s=0}^{t-1} \ell_r(o_s; x_s, \tau_{0,s}, \tau_{1,s}) +  \mathcal{O} \left( \log \frac{\mathcal{N}_{[]}(1/t, \mathcal{L}_r, L_\infty)}{\delta} \right)  \\
		=&  \sum_{s=0}^{t-1} \log \frac{\mathbb{P}_{r^*}(o_s\mid x_s, \tau_{0,s}, \tau_{1,s})}{\mathbb{P}_{r}(o_s\mid x_s, \tau_{0,s}, \tau_{1,s})} +  \mathcal{O} \left( \log \frac{\mathcal{N}_{[]}(1/t, \mathcal{L}_r, L_\infty)}{\delta} \right) \\
		\leq&  \sum_{s=0}^{t-1} \log \frac{\mathbb{P}_{\hat{r}}(o_s\mid x_s, \tau_{0,s}, \tau_{1,s})}{\mathbb{P}_{r}(o_s\mid x_s, \tau_{0,s}, \tau_{1,s})}  + \mathcal{O} \left( \log \frac{\mathcal{N}_{[]}(1/t, \mathcal{L}_r, L_\infty)}{\delta} \right) \\
		\leq& \zeta_r^t + \mathcal{O} \left( \log \frac{\mathcal{N}_{[]}(1/t, \mathcal{L}_r, L_\infty)}{\delta} \right)  \leq \mathcal{O} \left( \log \frac{\mathcal{N}_{[]}(1/t, \mathcal{L}_r, L_\infty)}{\delta} \right),
	\end{align*}
	where the forth step follows from $\hat{r} = \mathop{\arg\max}_{r\in \mathcal{G}_r} \sum_{s=0}^{t-1} \mathbb{P}_{r} (o_s\mid x_s, \tau_{0,s}, \tau_{1,s})$, the last step follows by setting $\zeta_r^t \leq \mathcal{O} \left( \log \mathcal{N}_{[]}(1/t, \mathcal{L}_r, L_\infty)/\delta \right)$. This completes the proof of Lemma \ref{lemma:unformart}.
	\end{proof}
	
    \section*{B Proofs for Offline PbKd}
    \hypertarget{proofB}{}
    \begin{proof}[Proof of Theorem \ref{thm:subopt}]
    	We denote $r_{\hat{\pi}} = \mathop{\arg\max}_{r\in\mathcal{R}(\mathcal{D}^{\rm pref}_{\rm off})} {J(\pi_E, r^*) - J(\hat{\pi}, r^*)}$ and $r_{\pi^*} = \mathop{\arg\max}_{r\in\mathcal{R}(\mathcal{D}^{\rm pref}_{\rm off})} {J(\pi_E, r^*) - J(\hat{\pi}, r^*)}$. Then, under Assumption \ref{asp:rea}, we can bound the suboptimality of $\hat{\pi}$ as follows:
    	\begin{align*}
    		&J(\pi^*, r^*) - J(\hat{\pi}, r^*) \\
    		=& J(\pi_E, r^*) - J(\hat{\pi}, r^*) - \left( J(\pi_E, r^*) - J(\pi^*, r^*) \right) \\
    		\leq& J(\pi_E, r^*) - J(\hat{\pi}, r^*) - \left(J(\pi_E, r_{\hat{\pi}}) - J(\hat{\pi}, r_{\hat{\pi}} )\right) \\
    		& - \left( J(\pi_E, r^*) - J(\pi^*, r^*) - \left( J(\pi_E, r_{\pi^*}) - J(\pi^*, r_{\pi^*}) \right) \right)\\
    		\leq& J(\pi_E, r_{\pi^*}) - J(\pi^*, r_{\pi^*}) - \left( J(\pi_E, r^*) - J(\pi^*, r^*) \right) \\
    		=& \mathbb{E}_{x \sim d_0,\tau_0 \sim \pi_E \mid x, \tau_1 \sim \pi^* \mid x} \left[ r_{\pi^*}(\tau_0) - r_{\pi^*}(\tau_1) - \left( r^*(\tau_0) - r^*(\tau_1) \right) \right] \\
    		\leq& C_r(\mathcal{G}_r, \pi_E, \pi^*) \sqrt{\mathbb{E}_{x \sim d_0, \tau_0 \sim \mu_0 \mid x, \tau_1 \sim \mu_1 \mid x}\left[ \vert r_{\pi^*}(\tau_0) - r_{\pi^*}(\tau_1) - \left( r^*(\tau_0) - r^*(\tau_1) \right) \vert^2 \right]},
    	\end{align*}
    	where the second step follows from $\hat{\pi} = \mathop{\arg\min}_{\pi \in \Pi_s} \max_{r \in \mathcal{R}(\mathcal{D}^{\rm pref}_{\rm off})} J(\pi_E,r) - J(\pi, r)$, the third step follows from the fact that $r_{\hat{\pi}} = \mathop{\arg\max}_{r \in \mathcal{R}(\mathcal{D}^{\rm pref}_{\rm off})} J(\pi_E, r) - J(\hat{\pi}, r)$, the fifth step comes from the definition of $C_r(\mathcal{G}_r, \pi_E, \pi^*)$ in Definition \ref{def:concencoefoffline}. We can further bound the above result with the following two steps.
    	
    	\noindent\textbf{Step 1.} Bounding $\sqrt{\mathbb{E}_{x \sim d_0, \tau_0 \sim \mu_0 \mid x, \tau_1 \sim \mu_1\mid x}\left[ \vert r_{\pi^*}(\tau_0) - r_{\pi^*}(\tau_1) - \left( r^*(\tau_0) - r^*(\tau_1) \right) \vert^2 \right]}$:
    	
    	Using Lemma \ref{lemma:l1logeexp} by setting the trajectory generation distributions $\pi_0 = \mu_0$ and $\pi_1 = \mu_1$, we have for any $r \in \mathcal{G}_r$,
	    \begin{align*}
	    \begin{split} \label{eq:lemma1eq1}
		&\mathbb{E}_{x \sim d_0, \tau_0 \sim \mu_0 \mid x, \tau_1 \sim \mu_1 \mid x}\left[ \vert r(x, \tau_0) - r(x, \tau_1) - \left( r^*(x, \tau_0) - r^*(x, \tau_1) \right) \vert^2 \right] \\
		\leq & -2 \kappa^2 \log \mathbb{E}_{x \sim d_0, \tau_0 \sim \mu_0 \mid x, \tau_1 \sim \mu_1 \mid x \atop o \sim \mathbb{P}_{r^*}(\cdot \mid x, \tau_0, \tau_1)} \exp\left( \ell_r(o; x, \tau_0, \tau_1) \right)
	   \end{split}
      \end{align*}
      
      Using Lemma \ref{lemma:uniformiid}, we have for any $\delta \in (0,1]$, with probability at least $1-\delta$,
      \begin{align*}
      	\forall r \in \mathcal{G}_r,\quad &\mathbb{E}_{x \sim d_0, \tau_0 \sim \mu_0 \mid x, \tau_1 \sim \mu_1 \mid x}\left[ \vert r(x, \tau_0) - r(x, \tau_1) - \left( r^*(x, \tau_0) - r^*(x, \tau_1) \right) \vert^2 \right] \\
      	\leq& \mathcal{O} \left( \kappa^2 \frac{\log\left( \mathcal{N}_{[]}(1/N, \mathcal{L}_r, L_\infty)/ \delta \right)}{N}  \right)
      \end{align*} 
      
      By setting $r = r_{\pi^*}$, we have for any $\delta \in (0,1]$, with probability at least $1-\delta$,
      \begin{align}
      	\begin{split}
      	&\sqrt{\mathbb{E}_{x \sim d_0, \tau_0 \sim \mu_0 \mid x, \tau_1 \sim \mu_1\mid x}\left[ \vert r_{\pi^*}(\tau_0) - r_{\pi^*}(\tau_1) - \left( r^*(\tau_0) - r^*(\tau_1) \right) \vert^2 \right]} \\
      	\leq& \mathcal{O} \left( \kappa \sqrt{\frac{\log\left( \mathcal{N}_{[]}(1/N, \mathcal{L}_r, L_\infty)/ \delta \right)}{N}}  \right)
      	\end{split}
      \end{align}
      
      Furthermore, if Assumption \ref{asp:linearbound} holds, we can use Proposition \ref{prop:bracketing} to bound the linear case of RM as follows.
      \begin{align}
      	 \begin{split}
      	 	&\sqrt{\mathbb{E}_{x \sim d_0, \tau_0 \sim \mu_0 \mid x, \tau_1 \sim \mu_1\mid x}\left[ \vert r_{\pi^*}(\tau_0) - r_{\pi^*}(\tau_1) - \left( r^*(\tau_0) - r^*(\tau_1) \right) \vert^2 \right]} \\
      	 	\leq&\mathcal{O} \left( (1+e^{2B})\sqrt{\frac{d\log(BN)+\log(1/\delta)}{N}}  \right)
      	 \end{split}
      \end{align}
      
      \noindent\textbf{Step 2.} Bounding $C_r(\mathcal{G}_r, \pi_E, \pi^*)$ under Assumption \ref{asp:linearbound}:
      
      We denote the covariance matrix corresponding to the expected difference between feature vectors of preference pairs as follows.
    	\begin{align*}
    		\Sigma = \frac{\lambda}{B}I + \mathbb{E}_{x \sim d_0, \tau_0 \sim \mu_0 \mid x, \tau_1 \sim \mu_1 \mid x} \left[ \left( \phi(x, \tau_0) - \phi(x, \tau_1) \right) \left( \phi(x, \tau_0) - \phi(x, \tau_1) \right)^\top \right]
    	\end{align*}
    	
    	Let $\lambda$ denote a positive constant satisfying: $0 < \lambda \leq \mathbb{E}_{x \sim d_0, \tau_0 \sim \mu_0 \mid x, \tau_1 \sim \mu_1 \mid x} \left[ \vert r^*(\tau_0) - r^*(\tau_1) - \left( r(\tau_0) - r(\tau_1) \right) \vert^2 \right]$ and denote $\phi(x, \pi) = \mathbb{E}_{\tau \sim \pi \mid x} [ \phi(x,\tau) ]$. We assume that $r^*$ can be represented as $r^*(x, \tau) = (\theta^*)^\top\phi(x,\tau)$. The concentrability coefficient can be bounded $C_r(\mathcal{G}_r, \pi_E, \pi^*)$ as:
    	\begin{align*}
    		C_r(\mathcal{G}_r, \pi_E, \pi^*) =& \sup_{r \in \mathcal{G}_r} \frac{\mathbb{E}_{x \sim d_0, \tau_0 \sim \pi_E \mid x, \tau_1 \sim \pi^* \mid x}\left[ r^*(x, \tau_0) - r^*(x, \tau_1) - \left( r(x, \tau_0) - r(x, \tau_1) \right)\right]}{\sqrt{\mathbb{E}_{x \sim d_0, \tau_0 \sim \mu_0 \mid x, \tau_1 \sim \mu_1 \mid x} \left[ \vert r^*(x, \tau_0) - r^*(x, \tau_1) - \left( r(x, \tau_0) - r(x, \tau_1) \right) \vert^2 \right]}} \\
    		\leq& \sqrt{2}\sup_{r \in \mathcal{G}_r} \frac{ \mathbb{E}_{x \sim d_0, \tau_0 \sim \pi_E \mid x, \tau_1 \sim \pi^* \mid x}\left[ r^*(x, \tau_0) - r^*(x, \tau_1) - \left( r(x, \tau_0) - r(x, \tau_1) \right)\right]}{\sqrt{ \lambda + \mathbb{E}_{x \sim d_0, \tau_0 \sim \mu_0 \mid x, \tau_1 \sim \mu_1 \mid x} \left[ \vert r^*(x, \tau_0) - r^*(x, \tau_1) - \left( r(x, \tau_0) - r(x, \tau_1) \right) \vert^2 \right]}} \\
    		\leq& \sqrt{2} \sup_{r \in \mathcal{G}_r} \frac{\left\vert (\theta^*-\theta)^\top \mathbb{E}_{x \sim d}\left[\phi(x, \pi_E) - \phi(x, \pi^*)\right]\right\vert}{\sqrt{(\theta^* - \theta)^\top \Sigma (\theta^* - \theta)}} \\
    		\leq& \sqrt{2} \sup_{r \in \mathcal{G}_r} \frac{\left\Vert \theta^*-\theta \right\Vert_{\Sigma} \left\Vert \phi(x, \pi_E) - \phi(x, \pi^*) \right\Vert_{\Sigma^{-1}}}{\sqrt{(\theta^* - \theta)^\top \Sigma (\theta^* - \theta)}} \\
    		=&\sqrt{2} \left\Vert \phi(x, \pi_E) - \phi(x, \pi^*) \right\Vert_{\Sigma^{-1}},
    	\end{align*}
    	where the forth inequality follows from the generalized Cauchy-Schwarz inequality. 
    \end{proof}
    
    Following from the results of {\bf Step 1}, we have for any $\delta \in (0,1]$, with probability at least $1-\delta$,
    \begin{align*}
    	J(\pi^*, r^*) - J(\hat{\pi}, r^*) \leq  \mathcal{O} \left( \kappa C_r(\mathcal{G}_r, \pi_E, \pi^*) \sqrt{\frac{\log\left( \mathcal{N}_{[]}(1/N, \mathcal{L}_r, L_\infty)/ \delta \right)}{N}}  \right),
    \end{align*}
    which completes the proof of the first conclusion in Theorem \ref{thm:subopt}.
    
    The results of {\bf Step 2} further give:
    \begin{align*}
    	J(\pi^*, r^*) - J(\hat{\pi}, r^*) \leq \mathcal{O} \left( (1+e^{2B}) \left\Vert \phi(x, \pi_E) - \phi(x, \pi^*) \right\Vert_{\Sigma^{-1}} \sqrt{\frac{d\log(BN)+\log(1/\delta)}{N}}  \right),
    \end{align*}
    which completes the proof of the second conclusion in Theorem \ref{thm:subopt}.
    
	\section*{C Proofs for Online PbKd}
	\hypertarget{proofC}{}
	\begin{proof}[Proof of Theorem \ref{thm:regret}]
		We denote $r_{\pi^t} = \mathop{\arg\max}_{r\in\mathcal{R}(\mathcal{D}^{\rm pref}_{\rm off})} {J(\pi_E, r^*) - J(\pi^t, r^*)}$ and $r_{\pi^*} = \mathop{\arg\max}_{r\in\mathcal{R}(\mathcal{D}^{\rm pref}_{\rm off})} {J(\pi_E, r^*) - J(\hat{\pi}, r^*)}$. Then, under Assumption \ref{asp:rea}, we have a regret bound as follows:
		\begin{align*}
			&\sum_{t=1}^T \left( J(\pi^*, r^*) - J(\pi^t, r^*) \right) \\
			=& \sum_{t=1}^T \left(J(\pi_E, r^*) - J(\pi^t, r^*)\right) - \sum_{t=1}^T \left( J(\pi_E, r^*) - J(\pi^*, r^*) \right) \\
			\leq& \sum_{t=1}^T \left(J(\pi_E, r^*) - J(\pi^t, r^*)\right) - \sum_{t=1}^T \left(J(\pi_E, r_{\pi^t}) - J({\pi}^t, r_{\pi^t} )\right) \\
			& - \left( \sum_{t=1}^T \left( J(\pi_E, r^*) - J(\pi^*, r^*) \right) - \sum_{t=1}^T \left( J(\pi_E, r_{\pi^*}) - J(\pi^*, r_{\pi^*}) \right) \right) \\
			\leq& \sum_{t=1}^T \left( J(\pi_E, r_{\pi^*}) - J(\pi^*, r_{\pi^*})\right) - \sum_{t=1}^T \left( J(\pi_E, r^*) - J(\pi^*, r^*) \right) \\
			=& \sum_{t=1}^T \mathbb{E}_{x \sim d_0, \tau_0 \sim \pi_E \mid x, \tau_1 \sim \pi^* \mid x} \left[ r_{\pi^*}(x, \tau_0) - r_{\pi^*}(x, \tau_1) - \left( r^*(x, \tau_0) - r^*(x, \tau_1) \right) \right] \\
			\leq& \sum_{t=1}^T C_r^t(\mathcal{G}_r, \pi_E, \pi^*) \sqrt{\sum_{s=0}^{t-1} \mathop{\mathbb{E}}\limits_{x_s \sim d_0, \tau_{0,s} \sim \pi_E\mid x_s, \tau_{1,s} \sim \pi^s \mid x_s} \!\! \left[ \vert r_{\pi^*}(x, \tau_{0,s}) - r_{\pi^*}(x, \tau_{1,s}) - \left( r^*(x, \tau_{0,s}) - r^*(x, \tau_{1,s}) \right) \vert^2 \right]},
		\end{align*}
		where the second step follows from the fact that $\pi^t = \mathop{\arg\min}_{\pi \in \Pi} \max_{r \in \mathcal{R}(\mathcal{D}^{\rm pref}_{t})}$, the third step comes from the fact that $r_{\pi^t} = \mathop{\arg\max}_{r\in\mathcal{R}(\mathcal{D}^{\rm pref}_{t})} {J(\pi_E, r^*) - J(\pi^t, r^*)}$. The last step follows from the definition of $ C_r^t(\mathcal{G}_r, \pi_E, \pi^*)$ in Definition \ref{def:concoeonline}. We can further bound the above result with the following two steps.
		
		\noindent\textbf{Step 1.} Bounding $\sqrt{\sum_{s=0}^{t-1} \mathop{\mathbb{E}}\limits_{x_s \sim d_0, \tau_{0,s} \sim \pi_E\mid x_s, \tau_{1,s} \sim \pi^s \mid x_s} \!\! \left[ \vert r_{\pi^*}(x, \tau_{0,s}) - r_{\pi^*}(x, \tau_{1,s}) - \left( r^*(x, \tau_{0,s}) - r^*(x, \tau_{1,s}) \right) \vert^2 \right]}$:
		
		 We use Lemma \ref{lemma:l1logeexp} by iteratively setting $\pi_0 = \pi_E$, $\pi_1 = \pi^s$ over $s \in \{0,1,\ldots, t-1\}$. We have for any $r \in \mathcal{G}_r$,
		\begin{align*}
			&\sum_{s=0}^{t-1}\mathop{\mathbb{E}}\limits_{x_s \sim d_0, \tau_{0,s} \sim \pi_E\mid x_s, \tau_{1,s} \sim \pi^s \mid x_s} \!\! \left[ \vert r(x, \tau_{0,s}) - r(x, \tau_{1,s}) - \left( r^*(x, \tau_{0,s}) - r^*(x, \tau_{1,s}) \right) \vert^2 \right] \\
			\leq&  -2 \kappa^2 \sum_{s=0}^{t-1} \log \mathbb{E}_{x_s \sim d_0, \tau_{0,s} \sim \pi_E \mid x_s, \tau_{1,s} \sim \pi^s \mid x _s\atop o_s \sim \mathbb{P}_{r^*}(\cdot \mid x_s, \tau_{0,s}, \tau_{1,s})} \exp\left( \ell_r(o_s; x_s, \tau_{0,s}, \tau_{1,s}) \right)
		\end{align*}
		
		By Lemma \ref{lemma:unformart}, we have for any $\delta \in (0,1]$, with probability at least $1-\delta$,
		\begin{align*}
			\forall r \in \mathcal{G}_r,\ &\sum_{s=0}^{t-1}\mathop{\mathbb{E}}\limits_{x_s \sim d_0, \tau_{0,s} \sim \pi_E\mid x_s, \tau_{1,s} \sim \pi^s \mid x_s} \!\! \left[ \vert r(x, \tau_{0,s}) - r(x, \tau_{1,s}) - \left( r^*(x, \tau_{0,s}) - r^*(x, \tau_{1,s}) \right) \vert^2 \right] \\
			\leq&  \mathcal{O} \left( \kappa^2 \log \frac{\mathcal{N}_{[]}(1/t, \mathcal{L}_r, L_\infty)}{\delta} \right)
		\end{align*}
		
		Thus, by setting $r = r_{\pi^*}$, we have for any $\delta \in (0,1]$, with probability at least $1-\delta$, 
		\begin{align*}
		    &\sqrt{ \sum_{s=0}^{t-1}\mathop{\mathbb{E}}\limits_{x_s \sim d_0, \tau_{0,s} \sim \pi_E\mid x_s, \tau_{1,s} \sim \pi^s \mid x_s} \!\! \left[ \vert  r_{\pi^*}(x, \tau_{0,s}) -  r_{\pi^*}(x, \tau_{1,s}) - \left( r^*(x, \tau_{0,s}) - r^*(x, \tau_{1,s}) \right) \vert^2 \right]} \\
			\leq&  \mathcal{O} \left( \kappa \sqrt{ \log \frac{\mathcal{N}_{[]}(1/t, \mathcal{L}_r, L_\infty)}{\delta} } \right)
		\end{align*}
		
		 Furthermore, if Assumption \ref{asp:linearbound} holds, we can use Proposition \ref{prop:bracketing} to bound the linear case of RM as follows
		 \begin{align*}
		 	&\sqrt{ \sum_{s=0}^{t-1}\mathop{\mathbb{E}}\limits_{x_s \sim d_0, \tau_{0,s} \sim \pi_E\mid x_s, \tau_{1,s} \sim \pi^s \mid x_s} \!\! \left[ \vert  r_{\pi^*}(x, \tau_{0,s}) -  r_{\pi^*}(x, \tau_{1,s}) - \left( r^*(x, \tau_{0,s}) - r^*(x, \tau_{1,s}) \right) \vert^2 \right]} \\
		 	\leq&  \mathcal{O} \left( (1+e^{2B}) \sqrt{ d\log Bt + \log(1/\delta) } \right)
		 \end{align*}
		
		\noindent\textbf{Step 2.} Bounding $C_r^t(\mathcal{G}_r, \pi_E, \pi^*)$ under Assumption \ref{asp:linearbound}:
		
		Definding the concentrability coefficient for $t \in \{1,2,\ldots,T\}$ as follows:
		\begin{align*}
			\Sigma_t = \frac{\lambda_t}{B} I + \sum_{s=0}^{t-1} \mathbb{E}_{x_s \sim d_0, \tau_{0,s} \sim \pi_E \mid x_s, \tau_{1,s} \sim \pi^s \mid x_s} \left[ (\phi(x_s, \tau_{0,s}) - \phi(x_s, \tau_{0,s})) (\phi(x_s, \tau_{0,s}) - \phi(x_s, \tau_{0,s})) ^\top\right]
		\end{align*}
		
		Given a positive constant $\lambda_t$ for each $t \in \{1,2,\ldots, T\}$: $0 < \lambda_t \leq \sum_{s=0}^{t-1} \mathbb{E}_{x_s \sim d_0, \tau_{0,s} \sim \pi_E \mid x_s, \tau_{1,s} \sim \pi^s \mid x_s} \left[ \vert  r^*(x_s, \tau_{0,1}) - r^*(x_s, \tau_{1,s}) - \left( r(x_s, \tau_{0,s}) - r(x_s, \tau_{1,s}) \right) \vert^2 \right]$ and denote $\phi(x, \pi) = \mathbb{E}_{\tau \sim \pi \mid x} [ \phi(x,\tau) ]$ for any $\pi \in \Pi$. We assume that $r^*$ can be represented as $r^*(x, \tau) = (\theta^*)^\top\phi(x,\tau)$. The concentrability coefficient can be bounded $C_r^t(\mathcal{G}_r, \pi_E, \pi^*)$ as:
		\begin{align*}
			&C_r^t(\mathcal{G}_r, \pi_E, \pi^*) \\
			=& \sup_{r \in \mathcal{G}_r} \frac{\mathbb{E}_{x \sim d_0, \tau_0 \sim \pi_E \mid x, \tau_1 \sim \pi^* \mid x}\left[r^*(x, \tau_0) - r^*(x, \tau_1) - \left( r(x, \tau_0) - r(x, \tau_1) \right)\right]}{\sqrt{\sum_{s=0}^{t-1} \mathbb{E}_{x_s \sim d_0, \tau_{0,s} \sim \pi_E \mid x_s, \tau_{1,s} \sim \pi^s \mid x_s} \left[ \vert  r^*(x_s, \tau_{0,s}) - r^*(x_s, \tau_{1,s}) - \left( r(x_s, \tau_{0,s}) - r(x_s, \tau_{1,s}) \right) \vert^2 \right]}} \\
			\leq& \sqrt{2} \sup_{r \in \mathcal{G}_r} \frac{\mathbb{E}_{x \sim d_0, \tau_0 \mid x \sim \pi_E \mid x, \tau_1 \sim \pi^* \mid x}\left[ r^*(x, \tau_0) - r^*(x, \tau_1) - \left( r(x, \tau_0) - r(x, \tau_1) \right)\right]}{\sqrt{\lambda_t + \sum_{s=0}^{t-1} \mathbb{E}_{x_s \sim d_0, \tau_{0,s} \sim \pi_E \mid x_s, \tau_{1,s} \sim \pi^s \mid x_s} \left[ \vert  r^*(x_s, \tau_{0,s}) - r^*(x_s, \tau_{1,s}) - \left( r(x_s, \tau_{0,s}) - r(x_s, \tau_{1,s}) \right) \vert^2 \right]}} \\
			\leq& \sqrt{2} \sup_{r \in \mathcal{G}_r} \frac{\left\vert (\theta^*-\theta)^\top \mathbb{E}_{x \sim d_0}\left[\phi(x, \pi_E) - \phi(x, \pi^*)\right]\right\vert}{\sqrt{(\theta^* - \theta)^\top \Sigma_t (\theta^* - \theta)}} \\
			\leq& \sqrt{2} \sup_{r \in \mathcal{G}_r} \frac{\left\Vert \theta^*-\theta \right\Vert_{\Sigma_t} \left\Vert \phi(x, \pi_E) - \phi(x, \pi^*) \right\Vert_{\Sigma^{-1}_t}}{\sqrt{(\theta^* - \theta)^\top \Sigma_t (\theta^* - \theta)}} \\
			=&\sqrt{2} \left\Vert \phi(x, \pi_E) - \phi(x, \pi^*) \right\Vert_{\Sigma_t^{-1}},
		\end{align*}
		where the forth step uses the generalized Cauchy-Schwarz inequality.
		
		Following from the result of {\bf Step 1}, we have for any $\delta$, with probability at least $1-\delta$,
		\begin{align*}
			\sum_{t=1}^T \left( J(\pi^*, r^*) - J(\pi_t, r^*) \right) \leq   \mathcal{O} \left( \kappa \sum_{t=1}^T C_r^t(\mathcal{G}_r, \pi_E, \pi^*) \sqrt{ \log \frac{\mathcal{N}_{[]}(1/t, \mathcal{L}_r, L_\infty)}{\delta} } \right),
		\end{align*}
		which completes the proof of the first conclusion in Theorem \ref{thm:regret}.
		
		Using the results of {\bf Step 2}, we further have
		\begin{align*}
			\sum_{t=1}^T \left( J(\pi^*, r^*) - J(\pi_t, r^*) \right) \leq \mathcal{O} \left( (1+e^{2B}) \sum_{t=1}^T \left\Vert \phi(x, \pi_E) - \phi(x, \pi^*) \right\Vert_{\Sigma_t^{-1}} \sqrt{ d\log Bt + \log(1/\delta) } \right) 
		\end{align*}
		
		By the Cauchy-Schwarz inequality, we have
		\begin{align*}
			&\sum_{t=1}^T \left( J(\pi^*, r^*) - J(\pi^t, r^*) \right) \\
			\leq&   \mathcal{O} \left( (1+e^{2B})  \sqrt{ \sum_{t=1}^T \left\Vert \phi(x, \pi_E) - \phi(x, \pi^*) \right\Vert_{\Sigma_t^{-1}}^2 }   \sqrt{ \sum_{t=1}^T \left( d\log Bt + \log(1/\delta) \right) }  \right) 
		\end{align*}
		
		By designing of optimistic algorithms (e.g., \cite{xiong2024iterative}) to maximize the
		relative uncertainty w.r.t. the teacher policy, the policy $\pi^t$ at each step $t \in \{1,2,\ldots, T\}$ is selected to maximize the confidence-bound metric in the inverse covariance-weighted norm: $\pi^t =  \mathop{\arg\max}_{\pi \in \Pi} \left\Vert \phi(x, \pi_E) - \phi(x, \pi) \right\Vert_{\Sigma_t^{-1}}^2$, in such case, we have
		\begin{align*}
			\sum_{t=1}^T \left\Vert \phi(x, \pi_E) - \phi(x, \pi^*) \right\Vert_{\Sigma_t^{-1}}^2 \leq \sum_{t=1}^T \left\Vert \phi(x, \pi_E) - \phi(x, \pi^t) \right\Vert_{\Sigma_t^{-1}}^2 \leq 2d\log\left( 1 + \frac{TB}{d\lambda_1} \right),
		\end{align*}
		where the second step follows from the Elliptical Potential Lemma (refer to e.g.,  \cite{dani2008stochastic,rusmevichientong2010linearly,abbasi2011improved,xiong2024iterative}).
		
		Then, we have
		\begin{align*}
			&\sum_{t=1}^T \left( J(\pi^*, r^*) - J(\pi^t, r^*) \right) \\
			\leq&   \mathcal{O} \left( (1+e^{2B})  \sqrt{ d\log\left( 1 + \frac{TB}{d\lambda_1} \right) \sum_{t=1}^T \left( d\log(Bt) + \log(1/\delta) \right) }  \right),
		\end{align*}
		which completes the proof of the second conclusion in Theorem \ref{thm:regret}.
		\end{proof}
	
	\section*{D Practical Implementation of Online PbKD}
	\hypertarget{proofD}{}
	
     \noindent\textbf{Algorithm Description.}  
Algorithm~\ref{alg:onlinekdpra} presents a practical implementation of online PbKD through a gradient-based adversarial training procedure.  
Given a prompt dataset $\mathcal{D}_x$, we initialize the student model with weights pre-trained on outputs generated by the teacher model.  
At the beginning of each online training iteration $t \in \{1, 2, \ldots, T\}$, new preference data are collected by sampling trajectories from both the teacher policy and the current student policy. These new samples are added to the cumulative preference dataset from previous iterations.  
The reward model (RM) is then updated by maximizing the objective defined in Eq.~(\ref{eq:blackobjpra}) using stochastic gradient ascent (SGA). Simultaneously, the student policy is optimized to minimize the same objective using a policy gradient method, incorporating a clipping mechanism~\cite{schulman2017proximal} to improve training stability.  
Additionally, to encourage exploration during training, we incorporate an uncertainty-based term that maximizes deviation from the teacher policy, following the approach proposed by Xiong et al.~\cite{xiong2024iterative} and Ye et al.~\cite{ye2024online}.
	
	\begin{algorithm}[h] 
		\caption{Online PbKD via Gradient-Based Adversarial Training} \label{alg:onlinekdpra}
		\KwIn{ Initial empty preference dataset $\mathcal{D}^{\rm pref}_0$, student policy $\pi$ pretrained on teacher-generated outputs, teacher policy $\pi_E$, prompt dataset $\mathcal{D}_x$, reward model $r_{\theta}$ with trainable parameters $\theta$ and a feature extractor $\phi$.}
		\For{$t = 1,2, \ldots, T$}{
			Sample $N$ inputs from the prompt dataset $\mathcal{D}_x$, i.e., for each $n \in [N]$: $x_{t-1}^n \in \mathcal{D}_x$ and trajectories $\tau_{0,t-1}^n \sim \pi_{E} \mid x_{t-1}^n$ and $\tau_{1,t-1}^n \sim \pi^{t-1} \mid x_{t-1}^n$ \\
			Add the sample into the previous preference set: \\
			$\mathcal{D}^{\rm pref}_{t} \leftarrow \mathcal{D}^{\rm pref}_{t-1} \cup \{(o_{t-1}^n=1; x_{t-1}^n, \tau_{0,t-1}^n, \tau_{1,t-1}^n)\}_{n=1}^N$ \\
			\# Compute the $t$-iteration best policy under constraints on the $t$-iteration confidence set: \\
			Sample a $M$ inputs from the prompt dataset $\mathcal{D}_x$, i.e., for each $m \in [M]$, $x_m \in \mathcal{D}_x$ and trajectories $\tau_{0,m} \sim \pi_E \mid x_m$ and $\tau_{1,m} \sim \pi^{t-1} \mid x_m$ \\
				\# Optimize the RM to maximize the objective of Eq. (\ref{eq:blackobjpra}): \\ 
				$\theta \leftarrow \theta + \sum_{m=1}^M \nabla_\theta \left[ \left( r_{\theta}(x_m, \tau_{0,m}) - r_{\theta}(x_m, \tau_{1,m})\right) + \beta L_{r_\theta}(\mathcal{D}_t^{\rm pref}) \right]$\\
				\# Optimize the student policy to minimize the objective of Eq. (\ref{eq:blackobjpra}): \\ 
				$\pi^t \leftarrow \pi^{t-1} + \sum_{m=1}^M r(x_m, \tau_{1,m}) \nabla_\pi \min \left( \frac{\pi(\tau_{1,m} \mid x_m)}{\pi^{t-1}(\tau_{1,m} \mid x_m)} , {\rm clip}( \frac{\pi(\tau_{1,m} \mid x_m)}{\pi^{t-1}(\tau_{1,m} \mid x_m)}, 1+\epsilon, 1-\epsilon)  \right)$ \\
				\# Maximize the uncertainty term: \\
				$\pi^t \leftarrow \pi^{t} + \alpha \nabla_\pi \left\Vert \phi(x, \pi_E) - \phi(x, \pi) \right\Vert_{\hat{\Sigma}_t^{-1}}^2$
			}
		\KwOut{The sequence of optimized student policy $\{\pi^t\}_{t=1}^T$.}
	\end{algorithm}

	\section*{E Detailed Experimental Setup and Additional Experiments}
	\hypertarget{proofE1}{}
	\subsection*{E.1 Experimental Setup for Black-box Knowledge Distillation}
	\hypertarget{proofE1}{}
	\noindent\textbf{Datasets.}
    Following Chen et al. \cite{chen2024knowledgedistillationblackboxlarge}, we construct our training corpus by sampling 100,000 prompts from a combined pool of OpenOrca \cite{lian2023openorca} and Nectar \cite{starling2023} datasets, which together contain GPT-4-generated outputs. OpenOrca comprises diverse instruction-following tasks, while Nectar provides 7-wise comparisons, from which we retain only GPT-4 responses. Following Chen et al. \cite{chen2024knowledgedistillationblackboxlarge}, we also incorporate GPT-4-generated synthetic data based on standard benchmarks to further enrich the training set.
	For evaluation, we use multiple benchmarks: the complex reasoning dataset BBH \cite{suzgun2023challenging}, knowledge-based datasets (AGIEval \cite{zhong2024agieval}, ARC-Challenge \cite{clark2018think}, and MMLU \cite{hendryckstest2021}), and the mathematical reasoning dataset GSM8K \cite{cobbe2021training}. All the reported results of accuracy (\%) on these five datasets are the average across three random seeds of $\{10,20,30\}$ by a zero-shot decoding strategy. 
	
	\noindent\textbf{Offline Preference Data Collection.}
	To precollect offline preference data, we fine-tune several large language models (LLMs) on 10,000 instruction-following instances from the training split to generate diverse candidate completions. Specifically, we employ models from both the LLaMA2 \cite{touvron2023llama} and Qwen2 \cite{qwen2} families, including LLaMA2-7B/13B/70B-Chat, Qwen2-1.5B/7B/72B. For each prompt, multiple completions are sampled across models to increase diversity, and candidate pairs are formed for evaluation. To assess the quality of these model outputs, we use GPT-4 as an automatic evaluator. Each candidate response is independently rated using the following standardized prompt in Figure \ref{fig:prompt}, The numeric ratings assigned by GPT-4 are used as pseudo-labels for offline reward modeling or preference data construction. When comparing pairs, we select the higher-rated response as the preferred one. If ratings are tied, we either discard the pair or resolve ties via a second round of evaluation to ensure label consistency.
	
	\noindent\textbf{Hyperparameters.}
	We select GPT-4 \cite{achiam2023gpt} as the black-box teacher model, a powerful proprietary large language model (LLM). For the student model, we use a pretrained LLaMA2-7B \cite{touvron2023llama} model with the teacher-generated training data. The reward model is constructed by adding a linear layer on top of a frozen LLaMA-7B model. For offline PbKD, the trainable parameters are randomly initialized, whereas for online PbKD, they are pretrained on the offline preference data.
	We iterate 5 iterations of online PbKD in total and in each iteration, we randomly select 1,000 prompts from the training set to construct preference data using outputs from the current student policy and the teacher model. These are incrementally added to the preference dataset for reward model training in the current iteration. 
	
	\noindent\textbf{Baselines.}  
	We compare our method with the following baselines:
	\begin{itemize}
		\item \textbf{Vanilla Black-Box KD:} The student model is fine-tuned directly on the outputs generated by the black-box teacher without additional sampling or optimization.
		
		\item \textbf{Best-of-$N$ Sampling}~\cite{stiennon2020learning}: To improve output quality, $N = 10$ candidate responses are sampled from the teacher using nucleus sampling (top-p = 0.9, top-k = 40, temperature = 0.8), and the one with the highest score according to a learned reward model is selected as the supervision signal.
		
		\item \textbf{Proxy-KD}~\cite{chen2024knowledgedistillationblackboxlarge}: A state-of-the-art black-box KD method. It first trains a white-box proxy model (LLaMA2-70B) to align with the black-box teacher via preference optimization, and then distills knowledge from this proxy to the student using standard white-box KD techniques.
	\end{itemize}
	
	\hypertarget{proofE2}{}
	\subsection*{E.2 Experimental Setup for White-box Knowledge Distillation}
	\hypertarget{proofE2}{}
	\noindent\textbf{Datasets.}
    We evaluate white-box knowledge distillation on several well-studied instruction-following datasets \cite{gu2023minillm,ko2024distillm,jiaadversarial}. Specifically, we construct the training data from the {\tt databricks-dolly-15k} dataset \cite{dolly2023introducing}, randomly selecting 15,000 samples for training and evenly splitting 500 samples for validation and testing. 
	We evaluate the trained student model on five instruction-following benchmarks: DollyEval, SelfInst \cite{wang2022self}, VicunaEval \cite{chiang2023vicuna}, S-NI \cite{wang2022super}, and UnNI \cite{honovich2022unnatural}. 
	
   \noindent\textbf{Offline Preference Data Collection.}
   To obtain preference data for policy-based knowledge distillation (PbKD), we fine-tune several LLMs on 5,000 randomly selected instruction-following instances from the {\tt databricks-dolly-15k} training split to generate diverse candidate completions. Similar to the black-box setting, we sample multiple responses per prompt across models to encourage diversity. GPT-4 is used as an automatic evaluator to score and rank the completions using the standardized evaluation prompt as illustrated in Figure \ref{fig:prompt}. Candidate pairs are then formed based on relative scores, with the higher-rated response selected as preferred. Tied cases are either discarded or resolved via a second evaluation pass to maintain label reliability.
	
	\noindent\textbf{Hyperparameters.}
	We use a fine-tuned LLaMA2-13B model \cite{touvron2023llama} as the teacher and a fine-tuned TinyLLaMA-1.1B model \cite{zhang2024tinyllama} as the student. The $Q$-function is implemented by adding a linear layer on top of a frozen OpenLLaMA-3B model \cite{geng2023openllama}. Similar to the black-box setting, $Q$-function is pretrained with the precollected preference data for online MM PbKD. The training procedure of the $Q$-function follows the same setting as in the black-box scenario, where 1,000 additional preference samples are incrementally added in each of 3 iterations for online PbKD. 
	
	\begin{figure}[h] 
		\centering 
		\includegraphics[width=1.0\linewidth]{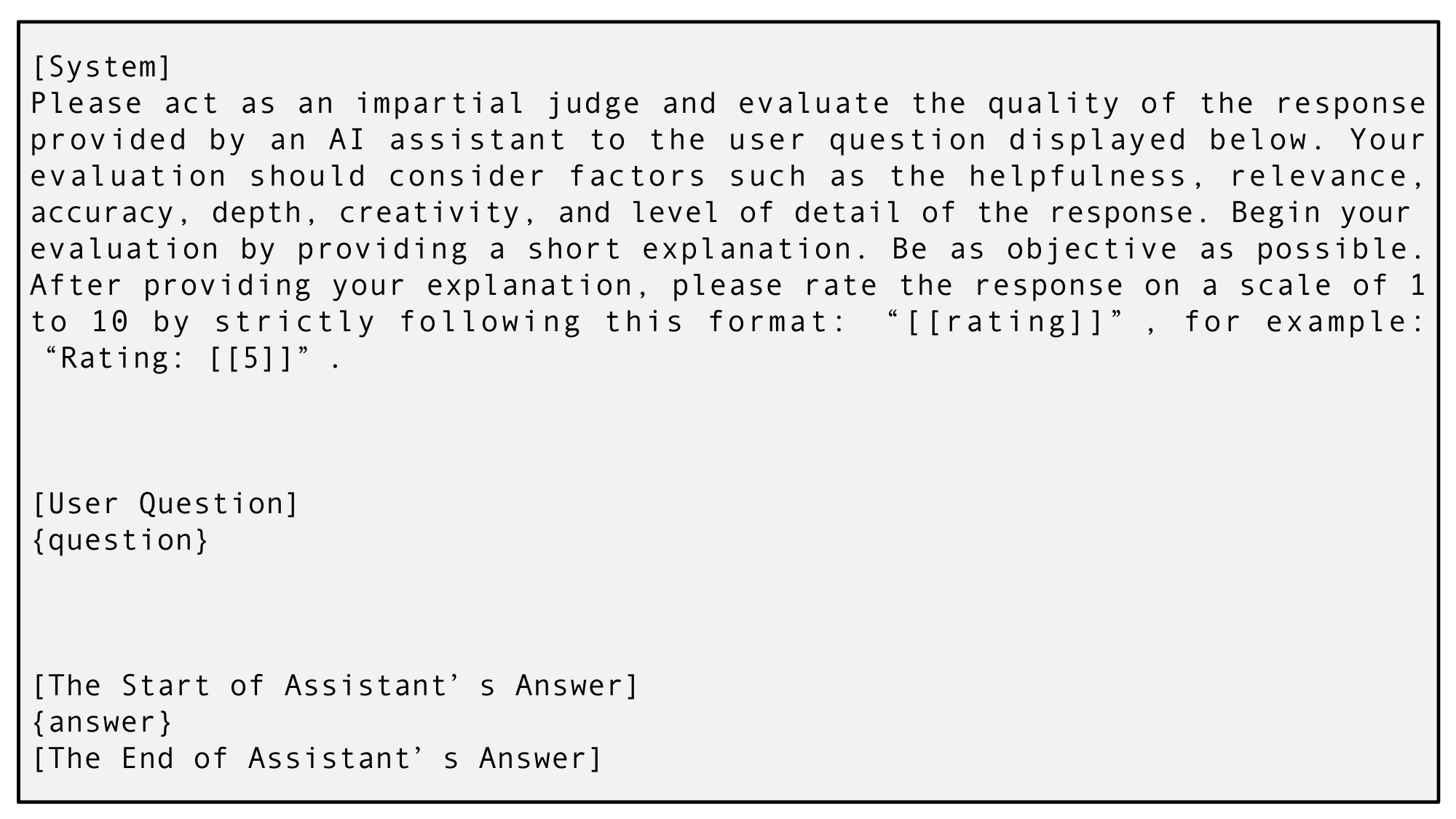} 
		\caption{The prompt template for GPT-4 feedback.} 
		\label{fig:prompt} 
	\end{figure}
	
	\noindent\textbf{Evaluation.}
	Evaluation metrics include ROUGE-L \cite{lin2004rouge}. Besides, we include the GPT-4 feedback scores \cite{zheng2023judging} as a supplementary evaluation metric, by asking GPT-4 to compare model-generated responses with the ground truth answers and raise 1-10 scores for both responses. We the prompt is largely followed \cite{ko2024distillm} and illustrated in Figure \ref{fig:prompt} and set the temperature $= 0.7$ to call the GPT-4 API. We report the ratio of the total score of model responses and ground truth answers. This metric is only applied to Dolly, SelfInst, and Vicuna. During the evaluation, we sample the responses from each model using a temperature of $0.8$ and three random seeds $\{10, 20, 30\}$.

\noindent\textbf{Baselines.}  
We compare our method against the following baselines:

\begin{itemize}
	\item \textbf{Black-Box KD:} Applies supervised fine-tuning (SFT) to the student using outputs generated by a black-box teacher.
	
	\item \textbf{KL}~\cite{hinton2015distilling}: Uses Kullback–Leibler divergence between the teacher and student output distributions on supervised datasets.
	
	\item \textbf{RKL}~\cite{lin2020autoregressive}: Employs reverse KL divergence on student-generated outputs, enabling on-policy learning.
	
	\item \textbf{MiniLLM}~\cite{gu2023minillm}: Enhances RKL with policy gradient methods to reduce training variance in auto-regressive generation.
	
	\item \textbf{GKD}~\cite{agarwal2024policy}: Utilizes Jensen–Shannon (JS) divergence under an on-policy formulation for stable distribution matching.
	
	\item \textbf{DistiLLM}~\cite{ko2024distillm}: Proposes adaptive off-policy training with skew KL divergence for improved sample efficiency.
	
	\item \textbf{AMMD}~\cite{jiaadversarial}: Employs adversarial moment-matching to reduce the imitation gap between teacher and student. For fair comparison, we apply Best-of-$N$ sampling (with $N=10$) during student training, following our black-box KD setting.
\end{itemize}
	 
	\hypertarget{proofE3}{}
	\subsection*{E.3  Additional Results Based on Qwen2}
	\begin{table*}[t!] 
		\centering
		\caption{Black-box KD on five datasets based on Qwen2 backbone. We format {\bf the best} and \underline{the second best} results.}
		\label{tbl:blackqwen}
		\scalebox{.9}{
			\begin{tabular}{lcccccc}
				\toprule
				{\bf Method}& BBH &AGIEval & ARC-Challenge & MMLU &GSM8K & {\bf Avg.} \\
				\midrule
				Teacher (GPT-4) & {\it 88.0} &{\it 56.4} & {\it 93.3} & {\it 86.4} & {\it 92.0} & {\it 83.2} \\ 
				\midrule
				\multicolumn{7}{c}{\it Qwen2-1.5B as the student backbone} \\ 
				\midrule
				Student (Qwen2-1.5B)  &36.0$_{\pm 0.3}$ &42.6$_{\pm 0.4}$&42.7$_{\pm 0.2}$&55.4$_{\pm 0.1}$&56.3$_{\pm 0.3}$&46.6 \\
				Vanilla Black-Box KD & 41.2$_{\pm 0.3}$ & 43.7$_{\pm 0.3}$ & 45.6$_{\pm 0.1}$ & 56.8$_{\pm 0.3}$ & 57.4$_{\pm 0.2}$ & 48.9 \\
				$\llcorner$ Best-of-$N$ & 43.5$_{\pm 0.2}$ & 43.8$_{\pm 0.2}$ & 50.7$_{\pm 0.1}$& 57.1$_{\pm 0.2}$ & 57.8$_{\pm 0.3}$ & 50.6  \\
				\midrule
				Offline PbKD  &47.4$_{\pm 0.2}$&44.7$_{\pm 0.3}$&50.4$_{\pm 0.3}$&58.5$_{\pm 0.2}$&60.5$_{\pm 0.2}$& 52.3 \\ 
				Online PbKD (1-iter) &\underline{50.4}$_{\pm 0.4}$&45.2$_{\pm 0.3}$&\underline{52.7}$_{\pm 0.3}$&59.4$_{\pm 0.2}$&60.5$_{\pm 0.3}$& 53.5 \\
				Online PbKD (3-iter) &49.7$_{\pm 0.3}$&\underline{46.7}$_{\pm 0.1}$&52.0$_{\pm 0.3}$&\underline{60.7}$_{\pm 0.3}$&\textbf{62.7}$_{\pm 0.4}$& 54.4 \\
				Online PbKD (5-iter) &\textbf{51.0}$_{\pm 0.6}$&\textbf{47.1}$_{\pm 0.2}$&\textbf{55.7}$_{\pm 0.4}$&\textbf{61.8}$_{\pm 0.2}$&\underline{61.4}$_{\pm 0.5}$& 55.4 \\
				\midrule
				\multicolumn{7}{c}{\it  Qwen2-7B as the student backbone} \\ 
				\midrule
				Student (Qwen2-7B)  &60.7$_{\pm 0.2}$&51.2$_{\pm 0.3}$&59.8$_{\pm 0.1}$&68.7$_{\pm 0.3}$&\underline{78.2}$_{\pm 0.3}$&63.7 \\
				Vanilla Black-Box KD & 61.4$_{\pm 0.2}$ & 50.0$_{\pm 0.2}$ & 61.5$_{\pm 0.4}$ & 69.5$_{\pm 0.3}$ & 75.5$_{\pm 0.3}$ & 63.6 \\
				$\llcorner$ Best-of-$N$ & 63.0$_{\pm 0.2}$ & 48.4$_{\pm 0.2}$ & 65.3$_{\pm 0.1}$& \textbf{70.4}$_{\pm 0.2}$ & 76.2$_{\pm 0.3}$ & 64.7  \\
				\midrule
				Offline PbKD  &63.7$_{\pm 0.2}$&50.1$_{\pm 0.2}$&64.2$_{\pm 0.3}$&67.4$_{\pm 0.5}$&74.7$_{\pm 0.3}$& 64.0 \\ 
				Online PbKD (1-iter) &64.3$_{\pm 0.3}$&\underline{51.7}$_{\pm 0.3}$&66.7$_{\pm 0.3}$&68.2$_{\pm 0.4}$&76.1$_{\pm 0.3}$& 65.4 \\
				Online PbKD (3-iter) &\underline{64.7}$_{\pm 0.2}$&50.8$_{\pm 0.1}$&\underline{68.7}$_{\pm 0.3}$&\underline{70.2}$_{\pm 0.4}$&\textbf{79.2}$_{\pm 0.5}$& 66.7 \\
				Online PbKD (5-iter) &\textbf{65.0}$_{\pm 0.4}$&\textbf{52.7}$_{\pm 0.3}$&\textbf{70.4}$_{\pm 0.2}$&69.7$_{\pm 0.3}$&77.0$_{\pm 0.3}$& 67.0 \\
				\bottomrule
		\end{tabular}}
	\end{table*}
	
 \noindent\textbf{Black-Box KD Results Based on Qwen2.}
 We evaluate black-box knowledge distillation (KD) using GPT-4 as the teacher model, with Qwen2-1.5B and Qwen2-7B as student models. The reward models are constructed based on Qwen2-1.5B and Qwen2-7B, respectively. The results, presented in Table~\ref{tbl:blackqwen}, exhibit trends similar to those observed in the LLaMA2-7B experiments.
 We compare our method against the vanilla black-box KD baseline, which fine-tunes the student model directly on outputs generated by the teacher. Additionally, we employ Best-of-$N$ sampling~\cite{stiennon2020learning} (with $N=10$, top-p = 0.9, top-k = 40, and temperature = 0.8) to enhance generation quality by selecting the highest-reward candidate according to a reward model trained on pre-collected preference data.
 Our method consistently outperforms both baselines across all five evaluation datasets, demonstrating the effectiveness of reward-guided optimization in improving student performance.
 Moreover, we perform an ablation study on the number of iterations in online PbKD. The results show that online PbKD consistently outperforms its offline counterpart, highlighting the benefit of iterative alignment through preference optimization. Performance continues to improve with more iterations, underscoring the value of reward-guided iterative distillation.

 \begin{table*}[t!] 
 	\centering
 	\caption{White-box KD on five instruction-following dataset based on Qwen2 backbone. We format {\bf the best} and \underline{the second best} results.}
 	\label{tbl:whiteboxqwen}
 	\scalebox{.8}{
 		\begin{tabular}{lcccccccc}
 			\toprule
 			\multirow{2}{*}{\bf Method}& \multicolumn{2}{c}{{DollyEval}} & \multicolumn{2}{c}{{SelfInst}} & \multicolumn{2}{c}{{VicunaEval}} & S-NI & UnNI \\
 			\cmidrule(lr){2-3}\cmidrule(lr){4-5}\cmidrule(lr){6-7}\cmidrule(lr){8-8}\cmidrule(lr){9-9}
 			& GPT-4 & R-L & GPT-4 & R-L & GPT-4 & R-L & R-L & R-L \\
 			\midrule
 			Teacher (Qwen2-7B) & {\it 58.5$_{\pm 0.5}$} & {\it 32.7$_{\pm 0.4}$} & {\it 53.4$_{\pm 0.4}$} & {\it 22.7$_{\pm 0.2}$} &{\it 50.5$_{\pm 0.2}$}  & {\it 24.0$_{\pm 0.3}$} & {\it 40.5$_{\pm 0.3}$} &{\it 37.3$_{\pm 0.4}$}  \\
 			\midrule
 			Student (Qwen2-0.5B) &34.6$_{\pm 0.2}$ &22.3$_{\pm 0.3}$&35.7$_{\pm 0.3}$&15.4$_{\pm 0.1}$&27.6$_{\pm 0.3}$&16.7$_{\pm 0.2}$&26.8$_{\pm 0.3}$&20.4$_{\pm 0.2}$ \\
 			Black-Box KD & 30.8$_{\pm 0.5}$ & 23.7$_{\pm 0.4}$ & 35.7$_{\pm 0.3}$  &14.3$_{\pm 0.2}$  & 26.7$_{\pm 0.6}$ & 15.4$_{\pm 0.3}$ & 27.2$_{\pm 0.3}$ & 21.7$_{\pm 0.2}$ \\
 			KL \cite{hinton2015distilling} & 32.5$_{\pm 0.3}$& 24.2$_{\pm 0.4}$ & 36.2$_{\pm 0.3}$& 15.0$_{\pm 0.1}$ & 28.2$_{\pm 0.1}$ & 17.2$_{\pm 0.3}$ & 27.4$_{\pm 0.3}$ & 22.3$_{\pm 0.2}$ \\
 			RKL \cite{lin2020autoregressive} & 32.5$_{\pm 0.7}$ &  24.5$_{\pm 0.3}$ & 37.5$_{\pm 0.5}$ & 16.0$_{\pm 0.2}$ & 30.1$_{\pm 0.4}$ & 17.5$_{\pm 0.0}$ & 25.2$_{\pm 0.7}$ & 20.7$_{\pm 0.2}$ \\
 			MiniLLM \cite{gu2023minillm} & 36.2$_{\pm 0.5}$ & 25.1$_{\pm 0.2}$ & 38.2$_{\pm 0.6}$ & 17.7$_{\pm 0.3}$ & 31.2$_{\pm 0.4}$ & 18.2$_{\pm 0.1}$ & 26.7$_{\pm 0.2}$ & 22.5$_{\pm 0.1}$ \\
 			GKD \cite{agarwal2024policy} & 35.7$_{\pm 0.6}$ & 23.6$_{\pm 0.4}$ & 41.2$_{\pm 0.8}$ & 18.6$_{\pm 0.4}$ & 29.2$_{\pm 0.3}$ & 17.4$_{\pm 0.2}$ & 28.8$_{\pm 0.1}$ & 21.5$_{\pm 0.5}$ \\
 			DistiLLM \cite{ko2024distillm} & 37.2$_{\pm 0.3}$ & 26.5$_{\pm 0.4}$ & 42.3$_{\pm 0.4}$ &  18.7$_{\pm 0.4}$ & 32.3$_{\pm 0.2}$ &17.4$_{\pm 0.3}$ & 27.4$_{\pm 0.2}$ & 23.8$_{\pm 0.2}$ \\
 			AMMD \cite{jiaadversarial} & 36.5$_{\pm 0.8}$ & 23.2$_{\pm 0.3}$ & 43.0$_{\pm 1.0}$ &18.9$_{\pm 0.5}$ & 30.2$_{\pm 0.4}$ & 16.7$_{\pm 0.4}$ & 28.0$_{\pm 0.5}$ & 25.6$_{\pm 0.1}$ \\
 			$\llcorner$ Best-of-$N$ &37.0$_{\pm 0.5}$&22.7$_{\pm 0.5}$&44.6$_{\pm 0.3}$&20.4$_{\pm 0.1}$&31.5$_{\pm 0.6}$&18.2$_{\pm 0.2}$&27.5$_{\pm 0.1}$&26.3$_{\pm 0.4}$ \\ 
 			\midrule
 			Offline MM PbKD  &38.4$_{\pm 0.7}$&25.2$_{\pm 0.6}$&44.5$_{\pm 0.2}$&19.7$_{\pm 0.1}$&32.5$_{\pm 0.7}$&18.5$_{\pm 0.2}$&28.5$_{\pm 0.3}$&27.6$_{\pm 0.2}$ \\ 
 			Online MM PbKD (1-iter) &\underline{39.7}$_{\pm 0.4}$&24.2$_{\pm 0.2}$&45.8$_{\pm 0.7}$&20.0$_{\pm 0.4}$&32.7$_{\pm 0.6}$&20.5$_{\pm 0.2}$&28.8$_{\pm 0.2}$&29.7$_{\pm 0.1}$ \\ 
 			Online MM PbKD (2-iter) &39.2$_{\pm 0.5}$&\underline{26.8}$_{\pm 0.3}$&\textbf{46.7}$_{\pm 0.7}$&\textbf{21.4}$_{\pm 0.2}$&\underline{33.5}$_{\pm 0.5}$&\underline{21.7}$_{\pm 0.3}$&\underline{30.2}$_{\pm 0.3}$&\underline{30.9}$_{\pm 0.1}$ \\ 
 			Online MM PbKD (3-iter) &\textbf{40.1}$_{\pm 0.7}$&\textbf{27.4}$_{\pm 0.3}$&\underline{46.2}$_{\pm 0.5}$&\underline{21.0}$_{\pm 0.1}$&\textbf{34.7}$_{\pm 0.5}$&\textbf{22.1}$_{\pm 0.2}$&\textbf{31.8}$_{\pm 0.3}$&\textbf{32.2}$_{\pm 0.2}$ \\ 
 			\bottomrule
 	\end{tabular}}
 \end{table*}

	\noindent\textbf{White-Box KD Results based on Qwen2.}
	We evaluate white-box knowledge distillation (KD) using Qwen2-7B as the teacher model, with Qwen2-0.5B as student models. The $Q$-function is constructed with a linear model on top of a fixed Qwen2-1.5B model.
	We present the white-box knowledge distillation (KD) results in Table \ref{tbl:whiteboxqwen}. Our approach is compared with Black-Box KD, where supervised fine-tuning (SFT) is applied to the student model using teacher-generated outputs. The results demonstrate that all white-box KD baselines generally outperform Black-Box KD, indicating that the token-level output probabilities from the white-box teacher provide richer information for implicit knowledge transfer. We further compare our method with several distribution matching baselines: KL \cite{hinton2015distilling}, which uses KL divergence on the supervised dataset; RKL \cite{lin2020autoregressive}, which applies reverse KL divergence on student-generated outputs; MiniLLM \cite{gu2023minillm}, which combines RKL divergence with policy gradient methods; GKD \cite{agarwal2024policy}, which utilizes JS divergence with an on-policy approach; and DistiLLM \cite{ko2024distillm}, which implements adaptive training for off-policy optimization of skew KL divergence. Our method surpasses these approaches by optimizing a more effective implicit imitation learning objective. Additionally, we compare with AMMD \cite{jiaadversarial}, which employs adversarial moment-matching to reduce the imitation gap between student and teacher policies. For fair comparison, we enhanced AMMD with a Best-of-$N$ sampling approach similar to our black-box KD experiments. However, AMMD's performance does not improve significantly with only Best-of-$N$ sampling and reward guidance. In contrast, our method enhances AMMD by incorporating reward guidance into an online preference-based KD learning process. The ablation study on online iterations reveals trends consistent with the black-box experiments, validating the effectiveness of iterative knowledge distillation with online preference optimization.

    \hypertarget{proofE4}{}
    \subsection*{E.4 Analysis Experiments}
    
    \begin{figure}[h]
    	\centering
    	\subfigure[Average performance of black-box KD across five datasets: BBH, AGIEval, ARC-Challenge, MMLU, and GSM8K.]{
    		\begin{minipage}[h]{0.45\textwidth}
    			\centering
    			\includegraphics[width=7cm]{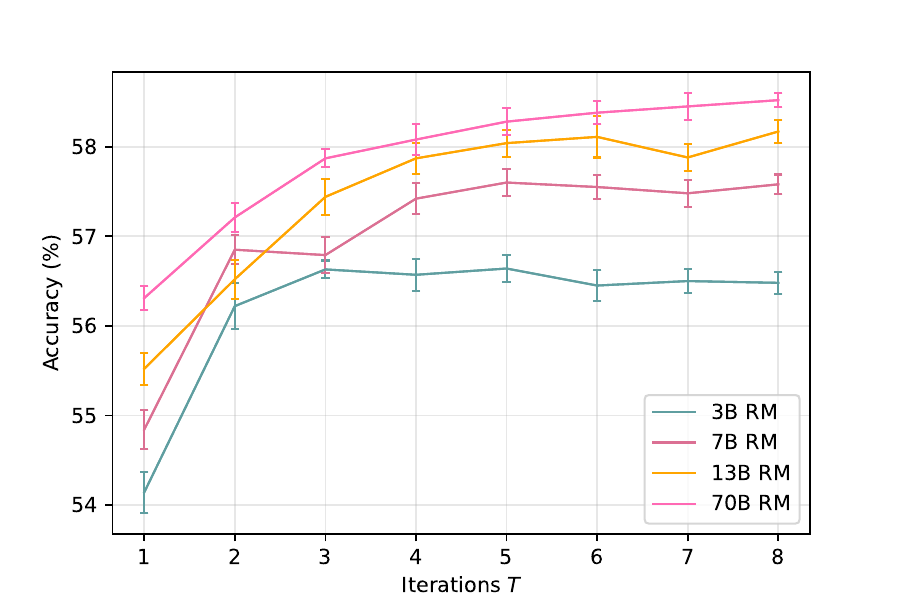}
    	\end{minipage}}
    	\hspace{0.1in}
    	\subfigure[Average ROUGE-L performance of white-box KD across five datasets: DollyEval, SelfInst, VicunaEval, S-NI, and UnNI.]{
    		\begin{minipage}[h]{0.45\textwidth}
    			\centering
    			\includegraphics[width=7cm]{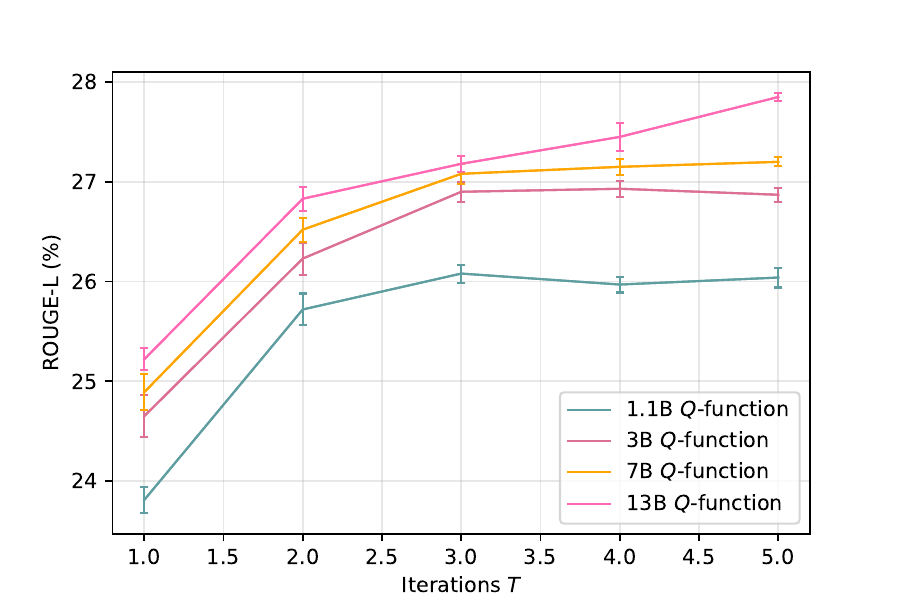}
    	\end{minipage}}
    	\caption{Impact of reward model ($Q$-Function) size on performance over online iterations.} \label{fig:rmsize}
    \end{figure}

    \begin{table*}[t!] 
    	\centering
    	\caption{Comparison with a combined RLHF and knowledge distillation method.}
    	\label{tbl:rlhfkd}
    	\scalebox{.9}{
    		\begin{tabular}{lcccccl}
    			\toprule
    			\multicolumn{7}{c}{\it  Black-box KD (Accuracy)} \\ 
    			\midrule
    			{\bf Method}& BBH &AGIEval & ARC-Challenge & MMLU &GSM8K&{\bf Avg.} \\   
    			\midrule
    			Offline PbKD  &50.7$_{\pm 0.3}$&40.3$_{\pm 0.3}$&70.5$_{\pm 0.6}$&53.9$_{\pm 0.3}$&52.7$_{\pm 0.4}$& 53.6  \\ 
    			Offline RLHF+KD &46.2$_{\pm 0.5}$&35.8$_{\pm 0.2}$&65.4$_{\pm 0.4}$&51.4$_{\pm 0.3}$&50.5$_{\pm 0.4}$& 49.9 ($-$3.7)  \\ 
    			Online PbKD  &54.7$_{\pm 0.6}$&44.7$_{\pm 0.3}$&73.1$_{\pm 0.5}$&57.3$_{\pm 0.6}$&58.2$_{\pm 0.4}$& 57.6  \\
    			Online RLHF+KD  &48.2$_{\pm 0.5}$&37.5$_{\pm 0.4}$&68.7$_{\pm 0.6}$&52.0$_{\pm 0.7}$&53.0$_{\pm 0.6}$& 51.9 ($-$5.7)  \\
    			\midrule
    			\multicolumn{7}{c}{\it  White-box KD (ROUGE-L)} \\ 
    			\midrule
    			{\bf Method}& DollyEval & SelfInst & VicunaEval & S-NI &UnNI &{\bf Avg.}  \\  
    			\midrule
    			Offline MM PbKD  &29.0$_{\pm 0.4}$&22.1$_{\pm 0.2}$&21.7$_{\pm 0.4}$&38.2$_{\pm 0.6}$&38.3$_{\pm 0.4}$& 29.9 \\ 
    			Offline RLHF+KD  &25.5$_{\pm 0.5}$&17.3$_{\pm 0.6}$&17.9$_{\pm 0.5}$&30.5$_{\pm 0.5}$&28.6$_{\pm 0.7}$& 24.0 ($-$5.9) \\ 
    			Online MM PbKD  &31.3$_{\pm 0.5}$&23.7$_{\pm 0.2}$&23.5$_{\pm 0.4}$&39.7$_{\pm 0.6}$&40.2$_{\pm 0.4}$& 31.7 \\ 
    			Online RLHF+KD  &26.7$_{\pm 0.6}$&20.7$_{\pm 0.5}$&20.5$_{\pm 0.6}$&32.0$_{\pm 0.7}$&32.8$_{\pm 0.7}$& 26.5 ($-$5.2) \\ 
    			\bottomrule
    	\end{tabular}}
    \end{table*}

  \noindent\textbf{Effects of Min-Max Optimization.}
  To evaluate the effectiveness of our proposed min-max optimization framework for joint policy and reward learning, we compare it against a pipeline approach. In this baseline, a reward model (RM) is first trained using either offline preference data or online generated preference data. Subsequently, the student policy is optimized using an RLHF method, combined with a cross-entropy loss (for black-box KD) or a KL-divergence objective (for white-box KD), based on outputs generated by the teacher model. We refer to these baselines as Offline RLHF+KD and Online RLHF+KD, corresponding to the offline and online settings, respectively.
  As shown in Table \ref{tbl:rlhfkd}, our min-max optimization framework improves distributional robustness in preference-based KD, especially when training and evaluation data come from different distributions. In such cases, conventional RLHF methods tend to underperform because the RM trained on the training distribution fails to accurately capture the true reward function in the test distribution.

  \noindent\textbf{Impact of RM Size.}
  Figure~\ref{fig:rmsize} presents an analysis of how the size of the reward model (RM) in black-box KD and the $Q$-function in white-box KD affects the performance of our online knowledge distillation method. Specifically, we construct the RM in black-box KD using four model sizes: OpenLLaMA-3B, LLaMA2-7B, LLaMA2-13B, and LLaMA2-70B; and the $Q$-function in white-box KD using TinyLLaMA-1.1B, OpenLLaMA-3B, OpenLLaMA-7B, and OpenLLaMA-13B. 
  From Figure~\ref{fig:rmsize}, we observe that as the number of iterations increases, all RM sizes exhibit a similar trend of performance improvement. This indicates that our PbKD approach benefits from multiple iterations of online learning. Notably, smaller RMs (e.g., OpenLLaMA-3B in black-box KD and TinyLLaMA-1.1B in white-box KD) tend to converge faster but reach a lower performance plateau. In contrast, larger RMs achieve better performance with more iterations, suggesting that the higher capability of larger models facilitates more effective alignment with the teacher model in our online PbKD framework.

\end{document}